\documentclass{article}

\usepackage[main, final]{neurips_2026}

\usepackage{microtype}
\usepackage{graphicx}
\usepackage{subcaption}
\usepackage{booktabs} % for professional tables

\usepackage[utf8]{inputenc} % allow utf-8 input
\usepackage[T1]{fontenc}    % use 8-bit T1 fonts
\usepackage{hyperref}       % hyperlinks
\usepackage{url}            % simple URL typesetting
\usepackage{booktabs}       % professional-quality tables
\usepackage{amsfonts}       % blackboard math symbols
\usepackage{nicefrac}       % compact symbols for 1/2, etc.
\usepackage{microtype}      % microtypography
\usepackage{xcolor}         % colors

\usepackage{amsmath}
\usepackage{amssymb}
\usepackage{mathtools}
\usepackage{amsthm}

\theoremstyle{plain}
\newtheorem{theorem}{Theorem}[section]
\newtheorem{proposition}[theorem]{Proposition}
\newtheorem{lemma}[theorem]{Lemma}

\theoremstyle{definition}
\newtheorem{definition}[theorem]{Definition}
\newtheorem{assumption}[theorem]{Assumption}
\theoremstyle{remark}

\usepackage{amsmath}
\usepackage{amsthm}  % for proposition
\usepackage{adjustbox}
\usepackage{multicol,multirow}
\usepackage{hyperref}
\usepackage{url}

\usepackage{wrapfig}   % 文字环绕宏包
\usepackage{bm}  % \bm bold font for math
\usepackage{amssymb}  % \mathbb for set representation
\usepackage{bbm}  % \mathbbm 
\usepackage{xcolor}  % \textcolor{black}{...}

\usepackage{color}
\usepackage{colortbl}
\definecolor{mylightblue}{rgb}{0.8, 0.9, 1.0}
\title{Stabilizing Policy Optimization via Logits Convexity}

\author{%
  \textbf{Hongzhan Chen}${}^{1,2}$ \quad 
  \textbf{Tao Yang}${}^{3*}$ \quad 
  \textbf{Yuhua Zhu}${}^{1}$ \quad 
  \textbf{Shiping Gao}${}^{1}$ \\ 
  \textbf{Xiaojun Quan}${}^{4}$\thanks{Corresponding author} \quad
  \textbf{Ting Yao}${}^{3}$ \\
  ${}^1$School of Computer Science and Engineering, Sun Yat-sen University \\ ${}^2$ Shanghai Innovation Institute \quad ${}^3$ Wechat Search, Tencent Inc, China. \\
  ${}^4$ Shenzhen Loop Area Institute \\
  \texttt{chenhzh59@mail2.sysu.edu.cn} \quad \texttt{xiaojunquan@slai.edu.cn}
}% \vspace{-15pt}

\begin{document}

\maketitle

\begin{abstract}
  While reinforcement learning (RL) has been central to the recent success of large language models (LLMs), RL optimization is notoriously unstable, especially when compared to supervised fine-tuning (SFT). In this work, we investigate the stability gap between SFT and RL from a gradient-based perspective, and show that the convexity of the SFT loss with respect to model logits plays a key role in enabling stable training. Our theoretical analysis demonstrates that this property induces favorable gradient directionality during optimization. In contrast, Proximal Policy Optimization (PPO), a widely adopted policy gradient algorithm utilizing a clipped surrogate objective, lacks this stabilizing property.
  Motivated by this observation, we propose \textbf{Logits Convex Optimization (LCO)}, a simple yet effective policy optimization framework that aligns the learned policy with an optimal target derived from the original RL objective, leveraging logit-level convexity to ensure a well-behaved loss geometry.
  Extensive experiments across multiple model families show that our LCO framework consistently improves training stability and outperforms conventional RL methods on a broad range of benchmarks.
\end{abstract}

\section{Introduction}

Reinforcement learning (RL) has become a cornerstone for aligning large language models (LLMs) with human preferences  \citep{ouyang2022training,bai2022training} and enhancing complex capabilities such as reasoning \citep{guo2025deepseek,yang2025qwen3}.  Despite these advances, RL training often suffers from inherent instability \citep{rafailov2024direct}. Existing approaches attempt to address this issue through variance reduction in advantage estimation \citep{schulman2015high}, clipping strategies that constrain parameter updates \citep{schulman2017proximal,yu2025dapo}, and KL-based penalties that regulate policy shifts \citep{ouyang2022training,shao2024deepseekmath}. 
Although these methods mitigate instability to some extent, it remains a persistent challenge in LLM optimization, and its underlying causes are still not fully understood \citep{team2025ring,zhu2025carft}.

% \begin{wrapfigure}{r}{0.6\textwidth}
%     \centering
%     \vspace{-2mm}
%     \includegraphics[width=1.0\linewidth]{figures/introduction-part-1.pdf}
%     \vspace{-6mm}
%     \caption{\textbf{(a)} Gradient norm during training for SFT and PPO. \textbf{(b)} Pass@1 results of SFT and PPO on the MATH500 benchmark.}
%     \label{fig:introduction-part-1}
%     \vspace{-2mm}
% \end{wrapfigure}

In this work, we adopt a gradient-centric perspective and explore whether intrinsic properties of the loss landscape contribute to RL instability in LLM training. Under this lens, we find a contrast between RL and supervised fine-tuning (SFT). In widely used RL algorithms, such as Proximal Policy Optimization (PPO) \citep{schulman2017proximal}, the clipped surrogate objective (i.e., the importance sampling ratio) often produces volatile gradients, even when standard stabilization techniques like clipping or KL regularization are applied (Figure~\ref{fig:introduction} (a)). These fluctuations can trigger excessively large parameter updates, potentially causing irreversible training collapse (Figure~\ref{fig:introduction} (b)). In contrast, SFT typically follows a more stable optimization trajectory.
%\citep{wu2025generalization,he2025amft,liu2025uft}, highlighting 
This raises a fundamental question: \emph{What explains the superior training stability of SFT compared to RL methods?}

By investigating the geometric properties of the optimization landscape, we identify a property termed \textit{logits convexity}, defined as local convexity at the logits of LLMs. Our theoretical analysis in Section \ref{sec:lco-gradient-analysis} demonstrates that logits convexity facilitates favorable gradient behavior during optimization. Specifically, under the assumption of the near-linear property of over-parameterized models, this property guarantees that the parameter-space gradient remains directionally aligned with the path toward the near-optimal parameters, as established in Proposition \ref{prop:gradient-directionality}. This provides a guarantee that gradient descent is not misled by spurious stationary points that may arise from the loss landscape. While SFT loss exhibits logits convexity, which ensures well-behaved gradient updates, RL objectives such as PPO \citep{schulman2017proximal} lack this property and are therefore susceptible to highly turbulent gradient dynamics that undermine training stability.

% \begin{figure}
%     \centering
%     \vspace{-2mm}
%     \includegraphics[width=0.7\linewidth]{figures/introduction-part-1.pdf}
%     \vspace{-3mm}
%     \caption{\textbf{(a)} Gradient norm during training for SFT and PPO. \textbf{(b)} Pass@1 results of SFT and PPO on the MATH500 benchmark.}
%     \label{fig:introduction-part-1}
%     \vspace{-2mm}
% \end{figure}

\begin{wrapfigure}{r}{0.6\textwidth}
    \centering
    \vspace{-4mm}
    \includegraphics[width=1.0\linewidth]{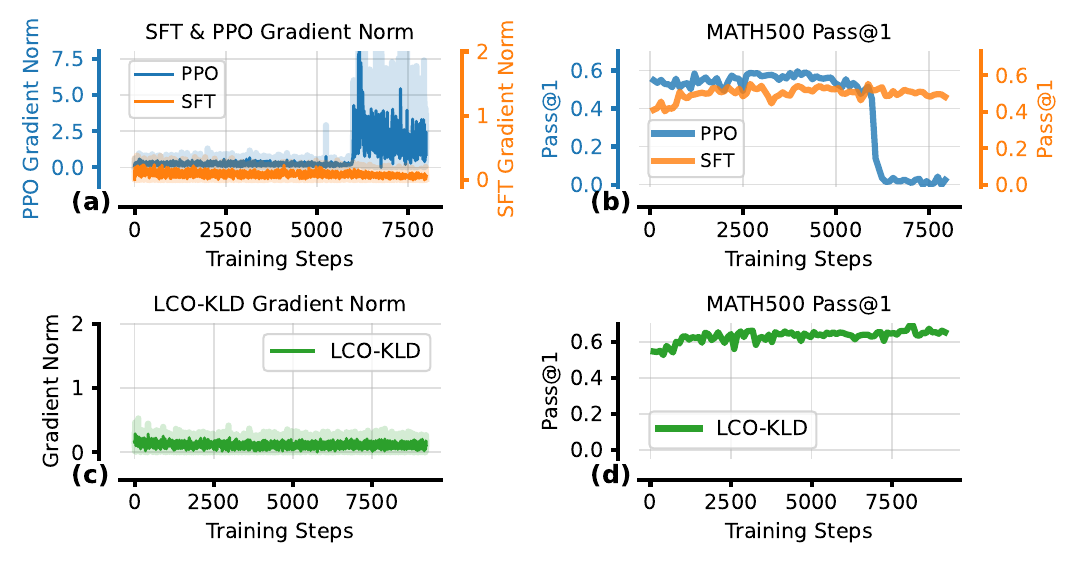}
    \vspace{-6mm}
    \caption{\textbf{(a)} Gradient norm during training for SFT and PPO. \textbf{(b)} Pass@1 results of SFT and PPO on the MATH500 benchmark. \textbf{(c)} Gradient norm during training for LCO-KLD. \textbf{(d)} Pass@1 results of LCO-KLD on MATH500.}
    \label{fig:introduction}
    \vspace{-2mm}
\end{wrapfigure}

% \begin{figure}[t]
%     \centering
%     \vspace{-2mm}
%     \includegraphics[width=0.1\linewidth]{figures/introduction-part-2.pdf}
%     \vspace{-3mm}
%     \caption{\textbf{(a)} Gradient norm during training for LCO-KLD. \textbf{(b)} Pass@1 results of LCO-KLD on MATH500.}
%     \label{fig:introduction-part-2}
%     \vspace{-5mm}
% \end{figure}

Building on these theoretical insights, we propose \textbf{Logits Convex Optimization (LCO)}, an RL optimization framework that preserves logits convexity and promotes stable training for LLMs. LCO reformulates the RL task as an optimal target matching problem, while remaining mathematically consistent with the original PPO objective that maximizes the expected advantage. 
Under this framework, we explore two primary approaches to improve the optimization landscape. First, we develop two direct alignment approaches (LCO-MSE and LCO-LCH) that target the optimal logits, ensuring strong convexity and enabling rapid convergence. Second, we examine a distribution-based approach (LCO-KLD) that aligns with the optimal distribution via forward KL divergence to ensure probabilistic consistency. By targeting the same optimal solution through a more well-behaved loss landscape, LCO produces reliable policy gradient updates and performance improvements (Figure \ref{fig:introduction} (c) and (d)). 

Empirical results across mathematical reasoning, reading comprehension, and instruction-following tasks show that LCO outperforms standard RL baselines in both stability and performance. Moreover, our analysis yields several key findings. 
\textbf{First}, we identify a primary source of training instability when using surrogate objectives: excessively large gradient norms arising from negative samples within non-convex loss regions. 
\textbf{Second}, we reveal that sampled actions with low probability can cause sudden spikes in gradient updates, which affect the stability of methods such as PPO.
\textbf{Third}, we show that the LCO objectives yield stable gradient updates that progressively diminish as training nears convergence, thereby mitigating instability in RL training.
% \vspace{-2mm}
\section{Preliminary}
% \vspace{-2mm}
\subsection{Problem Setup}
% \vspace{-1.5mm}
We model autoregressive text generation as a Markov decision process (MDP). At time step $t$, the state $s_t$ consists of the prompt tokens together with all tokens generated up to step $t$. An action $a_t \in \mathcal V$ corresponds to selecting the next token from the vocabulary $\mathcal{V}$.
Given state $s_t$ at time step $t$, the policy $\pi_\theta$ defines a distribution over actions: $\pi_{\theta}(a_t|s_t)=\frac{\exp z_{\theta}(s_t,a_t)}{\sum_{a'_t\in\mathcal{V}}\exp z_{\theta}(s_t,a'_t)}$
where $z_{\theta}(s_t, a_t)$ is the logit associated with action $a_t$ at state $s_t$, parameterized by $\theta$.
% \vspace{-1.5mm}
\subsection{Supervised Fine-Tuning}
% \vspace{-1.5mm}
Supervised fine-tuning (SFT) trains the language model to maximize the likelihood of target (ground-truth) tokens $a_t$ given the preceding context $s_t$. The objective is defined as:
\begin{equation}
\label{eq:sft-objective}
\mathcal{L}_{\text{SFT}} = - \mathbb{E}_t
\left[ \log \pi_\theta(a_t | s_t) \right],
\end{equation}
where the expectation $\mathbb{E}_t$ indicates the empirical average over a finite batch of samples.
% \vspace{-1.5mm}
\subsection{Proximal Policy Optimization}
% \vspace{-1.5mm}
Proximal Policy Optimization (PPO) \citep{schulman2017proximal} is a widely used policy-gradient method for reinforcement learning. A common PPO-style objective augments the expected advantage with a KL penalty that constrains the policy $\pi_\theta$ to remain close to a behavioral policy $\pi_{\text{old}}$:
\begin{equation}
    \label{eq:ppo-objective-with-kl-penalty}
\max_{\pi_\theta}\mathbb{E}_{a_t\sim\pi_\theta(\cdot|s_t)}[A(s_t,a_t)]-\beta\mathbb{D}_\text{KL}(\pi_\theta(\cdot|s_t)\|\pi_\text{old}(\cdot|s_t)),
\end{equation}
where $A(s_t, a_t)$ denotes the advantage function and $\beta > 0$ controls the strength of the KL regularization.
Rather than optimizing this KL-regularized objective directly, PPO introduces a clipped surrogate objective that heuristically constrains policy updates by limiting change in action probabilities between successive policies.
\begin{equation}
    \begin{aligned}
         \mathcal{L}_\text{PPO}=-\mathbb{E}_{t} [ \min(r_t(\theta)A(s_t,a_t), \text{clip}(r_t(\theta),1-\epsilon,1+\epsilon)A(s_t,a_t))],
    \end{aligned}
    \label{eq:ppo-objective}
\end{equation}
where $r_t(\theta)=\frac{\pi_\theta(a_t|s_t)}{\pi_{\text{old}}(a_t|s_t)}$, $\text{clip}(\cdot)$ is the clip function and the $\epsilon$ is the clip range. 
This clipped surrogate objective prevents excessively large policy updates, thereby improving training stability compared to standard policy gradient methods.

\begin{figure*}[t]
    \centering
    % \vspace{-12mm}
    \includegraphics[width=0.99\linewidth]{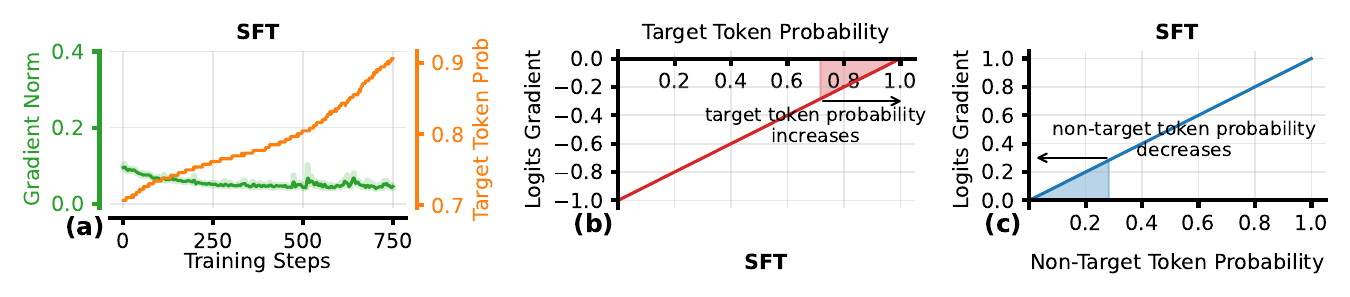}
    \vspace{-2mm}
    \caption{Training dynamics of SFT. \textbf{(a)} Gradient norm decreases as training progresses while target token probability on training sample increases. The magnitudes of both target token logit gradient \textbf{(b)} and the non-target token logit gradient \textbf{(c)} diminish as training progresses.}
    \label{fig:introduction-sft}
    % \vspace{-2mm}
\end{figure*}

\begin{figure*}[t]
    \centering
    % \vspace{-2mm}
    \includegraphics[width=0.99\linewidth]{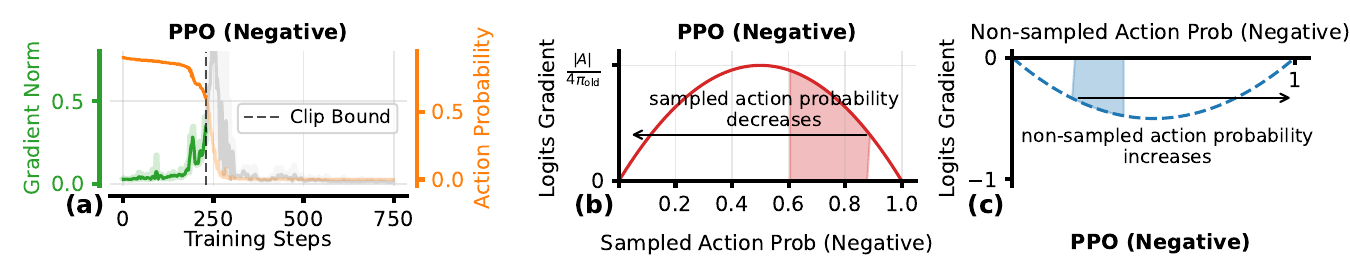}
    \vspace{-2mm}
    \caption{Training dynamics of PPO for actions with negative advantages. \textbf{(a)} Gradient norm increases before being clipped, while sampled action probabilities decrease. The magnitudes of both sampled action logit gradient \textbf{(b)} and non-sampled action logit gradient \textbf{(c)} increase before being clipped. The results show a correlation between gradients of parameter-space and logit-space.}
    \label{fig:introduction-ratio}
    % \vspace{-5mm}
\end{figure*}

% \vspace{-2mm}
\section{Gradient Dynamics}
% \vspace{-2mm}
In this section, we empirically analyze the distinct gradient dynamics of $\mathcal{L}_\text{SFT}$ and $\mathcal{L}_\text{PPO}$, and demonstrate how these dynamics influence logits gradients and training stability.
% \vspace{-1mm}
\subsection{Gradient Dynamics of SFT}
% \vspace{-1mm}
% We first consider the gradient of SFT loss $\mathcal{L}_\text{SFT}$ with respect to parameters $\theta$:
We analyze the gradient behavior of SFT starting from the standard negative log-likelihood objective in Equation \eqref{eq:sft-objective}. At a time step $t$, let $a_t$ denote the target (or ground-truth) token and $a'_t$ be an arbitrary token in the vocabulary $\mathcal{V}$. The derivative of $\mathcal{L}_\text{SFT}$ with respect to the logit $z_\theta(s_t,a'_t)$ is:
% \vspace{-1mm}
\begin{equation}
    \label{eq:sft-logit-gradient}
    \frac{\partial \mathcal{L}_\text{SFT}}{\partial z_\theta(s_t, a'_t)}=\pi_\theta(a'_t|s_t)-\mathbb{I}_{a_t=a'_t},
\end{equation}
where $\mathbb{I}_{a_t=a'_t}$ is the indicator (derivation is in Appendix \ref{sec:appendix-sft-logit-gradient-derivation}). The logit gradient of a target token $a_t$ is $\pi_\theta(a_t|s_t)-1$, which pushes the model to increase the target logit. While for a non-target token $a'_t$ ($a_t\neq a'_t$), the logit gradient equals $\pi_\theta(a'_t|s_t)$, which pushes the model to decrease the non-target logit. Figure \ref{fig:introduction-sft} (a) illustrates the overall gradient dynamics, while Figures \ref{fig:introduction-sft} (b) and (c) depict the logit gradient dynamics. During training, target token probabilities increase and gradient norms decrease. We observe a parallel trend in the logit gradients: as target token probabilities asymptotically increase and the probabilities of non-target tokens vanish, the corresponding logit gradient magnitudes diminish correspondingly, reflecting convergence. This behavior aligns with the intuition that as the model nears optimality, the parameter updates naturally become smaller.

% \vspace{-1mm}
\subsection{Gradient Dynamics of PPO}
\vspace{-0mm}
The clipping mechanism in PPO restricts parameter updates by nullifying the gradient when the ratio $r_t(\theta)$ moves outside a predefined interval. Specifically, the gradient is non-zero only when the following condition is satisfied: $(A(s_t,a_t) > 0 \text{ and } r_t(\theta)<1+\epsilon) \text{ or } (A(s_t,a_t)<0 \text{ and } r_t(\theta) > 1 - \epsilon)$. If this condition is not met, the objective enters the clipped region, and the gradient becomes zero. Consequently, we focus our analysis on the active region where the gradient is non-zero, and examine the dynamics of the gradient. At time step $t$, let $a_t$ be the sampled action and $a'_t$ be an arbitrary action in the vocabulary $\mathcal{V}$. Under this setting, the gradient of $\mathcal{L}_\text{PPO}$ with respect to the logit $z_\theta(s_t,a'_t)$ is given by (derivation is in Appendix \ref{sec:appendix-ppo-logit-gradient-derivation}):
% \vspace{-2mm}
\begin{equation}
    \label{eq:ppo-logit-gradient}
    \frac{\partial \mathcal{L}_\text{PPO}}{\partial z_\theta(s_t,a'_t)}=\frac{A(s_t,a_t)}{\pi_\text{old}(a_t|s_t)}\pi_\theta(a_t|s_t)(\pi_\theta(a'_t|s_t)-\mathbb{I}_{a_t=a'_t}).
\end{equation}
We focus on the case where $A(s_t,a_t)<0$ (for $A(s_t,a_t)>0$, see Figure \ref{fig:introduction-ratio-positive} in the Appendix). 
Figure \ref{fig:introduction-ratio} (a) shows that gradient norm exhibits an increase before being clipped. A similar phenomenon can also be observed in logit gradients (Figure \ref{fig:introduction-ratio} (b) and (c)). \textbf{A counterintuitive phenomenon emerges: as training progresses, the loss decreases while the gradient norms grow}. A comparison of the gradient norms for positive advantages reveals that updates are largely dominated by actions with negative advantages due to their disproportionately larger gradient magnitudes. This dominance results in a systemic reduction of action probabilities in subsequent sampling. Such gradient spikes typically manifest for actions with intermediate probabilities (e.g., near 0.5), triggering massive parameter updates that destabilize the training process. Furthermore, our analysis indicates that without the clipping mechanism, the gradient norm would exhibit initial increases followed by decreases, aligning closely with the dynamics of the logit gradients.
% \vspace{-2mm}
\section{On the Convexity of Logits}
% \vspace{-2mm}
Drawing upon the aforementioned analysis of gradient dynamics, we propose \textit{Logits Convex Optimization} (\textit{LCO}), a general training paradigm designed to stabilize the training of LLMs in reinforcement learning. The fundamental principle of LCO is to reformulate the complex RL task as a supervised alignment problem toward an optimal target derived from the original objective.
We start from the standard regularized RL objective in Eq. \eqref{eq:ppo-objective-with-kl-penalty}. Under this objective, the optimal policy $\pi^*$ that maximizes Equation \eqref{eq:ppo-objective-with-kl-penalty} at time step $t$ has the following form, as established in \citep{peng2019advantage}: 
% \vspace{-3mm}
\begin{equation}
        \label{eq:trpo-optimal-policy}
        \pi^*(a_t|s_t)=\frac{\pi_\text{old}(a_t|s_t)\exp\left[\frac{A(s_t,a_t)}{\beta}\right]}{\sum_{a'_t\in\mathcal{V}}\pi_\text{old}(a'_t|s_t)\exp\left[\frac{A(s_t,a'_t)}{\beta}\right]},
\end{equation}
where $\pi_\text{old}$ is the behavioral policy and $A(s_t,a_t)$ is the advantage function. A specific solution for the optimal logits $z^*(s_t,a_t)$ that induces $\pi^*$ is obtained by direct advantage-based adjustment:
% \vspace{-2mm}
\begin{equation}
        \label{eq:optimal-logits}
        z^*(s_t,a_t)=z_\text{old}(s_t,a_t)+\frac{A(s_t,a_t)}{\beta},
\end{equation}
where $z_\text{old}(s_t,a_t)$ is the logit of behavioral policy $\pi_\text{old}$ corresponding to the action $a_t$ and state $s_t$. We provide detailed derivation in Appendix \ref{sec:proof-of-trpo-optimal-policy}.
% \begin{proposition}
%     \label{prop:trpo-optimal-policy}
%     For a given behavioral policy $\pi_\text{old}$ and advantage function $A(s_t,a_t)$, the optimal policy $\pi^*$ that maximizes Equation \eqref{eq:ppo-objective-with-kl-penalty} at time step $t$ is:
%     \vspace{-1mm}
%     \begin{equation}
%         \label{eq:trpo-optimal-policy}
%         \pi^*(a_t|s_t)=\frac{\pi_\text{old}(a_t|s_t)\exp\left[\frac{A(s_t,a_t)}{\beta}\right]}{\sum_{a'_t\in\mathcal{V}}\pi_\text{old}(a'_t|s_t)\exp\left[\frac{A(s_t,a'_t)}{\beta}\right]}.
%     \end{equation}
%     A specific solution for the optimal logits $z^*(s_t,a_t)$ that induces $\pi^*$ is obtained by direct advantage-based adjustment:
%     \vspace{-1mm}
%     \begin{equation}
%         \label{eq:optimal-logits}
%         z^*(s_t,a_t)=z_\text{old}(s_t,a_t)+\frac{A(s_t,a_t)}{\beta},
%     \end{equation}
%     where $z_\text{old}(s_t,a_t)$ is the logit of behavioral policy $\pi_\text{old}$ corresponding to the action $a_t$ and state $s_t$.
% \end{proposition}
% \vspace{-3mm}
% \begin{proof}
%     See Appendix \ref{sec:proof-of-trpo-optimal-policy}.
% \end{proof}
% \vspace{-2mm}
The general objective of LCO is to fit the parameterized policy $\pi_\theta$ to these optimal targets (either the optimal distribution $\pi^*$ or the optimal logits $z^*$). Under this framework, we propose two distinct implementation strategies: regression-based alignment and distribution-based alignment.
% \vspace{-1mm}
\subsection{LCO Objectives}
% \vspace{-2mm}
\paragraph{Regression-based LCO Objectives} 
Following the characterization of the optimal logits in Equation \eqref{eq:optimal-logits}, we first define the LCO objective as the minimization of the discrepancy between the target optimal logits $z^*$ and logits $z_\theta$ of the parameterized policy $\pi_\theta$ at each time step. 
% To enforce this alignment through standard regression metrics, 我们提出两个LCO variants using MSE 和 Log-Cosh. we first apply the MSE and this yields the $\mathcal{L}_\text{LCO-MSE}$:
To enforce this alignment, we propose two variants of the LCO objective using standard regression metrics: MSE and log-cosh loss. The first variant applies MSE, yielding the $\mathcal{L}_\text{LCO-MSE}$:
% \vspace{-1mm}
\begin{equation}
    \label{eq:lco-mse-objective}
    \mathcal{L}_\text{LCO-MSE}=\mathbb{E}_t\left[\frac{1}{|\mathcal{V}|}\sum_{a_t\in\mathcal{V}}(z_\theta(s_t,a_t)-z^*(s_t,a_t))^2\right],
\end{equation}
where $\mathcal{V}$ is the vocabulary. Moreover, to provide a smoother optimization landscape that is more robust to outliers, we define the log-cosh \citep{saleh2022statistical} variant $\mathcal{L}_\text{LCO-LCH}$ as follows:
% \vspace{-1mm}
\begin{equation}
    \label{eq:lco-log-cosh-objective}
    \begin{aligned}
        \mathcal{L}_\text{LCO-LCH}= \mathbb{E}_t \left[\frac{1}{|\mathcal{V}|}\sum_{a_t\in\mathcal{V}}\log\cosh(z_\theta(s_t,a_t)-z^*(s_t,a_t))\right].
    \end{aligned}
\end{equation}
The $\log\cosh(\cdot)$ loss behaves like MSE for small errors but transitions to a linear penalty for large discrepancies, ensuring gradient stability even with noisy advantage estimates.
% \vspace{-2mm}
\paragraph{Distribution-based LCO Objective} 
While the regression-based losses above directly target logit values, we can also formulate the objective in terms of the optimal policy distribution. We then define $\mathcal{L}_\text{LCO-KLD}$ as the minimization of the forward KL divergence between the optimal policy $\pi^*$ and the parameterized policy $\pi_\theta$:
% \vspace{-0.5mm}
\begin{equation}
    \label{eq:lco-kld-objective}
    \mathcal{L}_\text{LCO-KLD}=\mathbb{E}_{t}\left[\sum_{a_t\in\mathcal{V}}\pi^*(a_t|s_t)\log\frac{\pi^*(a_t|s_t)}{\pi_\theta(a_t|s_t)}\right].
\end{equation}
Theoretically, since $\pi^*$ is the unique global maximizer of Equation \eqref{eq:ppo-objective-with-kl-penalty}, minimizing any LCO loss (e.g., $\mathcal{L}_\text{LCO-MSE}$, $\mathcal{L}_\text{LCO-LCH}$, or $\mathcal{L}_\text{LCO-KLD}$) is equivalent to solving the original RL problem, provided the parameterized policy $\pi_\theta$ has sufficient capacity to represent $\pi^*$ or $z^*$.
% \vspace{-2.0mm}
\subsection{Advantage Estimation} 
% \vspace{-2.0mm}
The LCO framework requires advantage signals $A(s_t, a_t)$ to construct the optimal target. Depending on the availability of feedback, we employ three strategies to estimate this advantage, ranging from sparse, sample-based signals to dense, distribution-level signals.
% \vspace{-2.5mm}
\paragraph{Sampled Action Sparse Estimation} In traditional reinforcement learning settings, such as PPO, advantage values are typically only available for the action $a_t$ actually sampled from the behavioral policy. For any other non-sampled action $a'_t$ in the vocabulary $\mathcal{V}$, the advantage signal is unknown. In this case, we adopt a sparse estimation strategy:
% \vspace{-1mm}
\begin{equation}
    \hat{A}(s_t, a'_t) =
    \begin{cases}
    A(s_t, a_t), & \text{if } a'_t = a_t, \\
    0, & \text{otherwise}.
    \end{cases}
\end{equation}
% \vspace{-4mm}
\paragraph{Log Probability-Based Dense Estimation} A more comprehensive approach provides training signals across the entire vocabulary $\mathcal{V}$. Building on the theoretical framework of \citet{ziebart2010modeling}, the advantage at the distribution level can be derived from the log probability assigned by an LLM $\phi$:
\begin{equation}
    \log\phi(a_t|s_t)=Q(s_t,a_t)-V(s_t),
\end{equation}
where $Q(s_t,a_t)$ is the Q-function and $V(s_t)$ is the value function. Using the soft Bellman equation, the advantage estimator is therefore defined as: $A(s_t,a_t):=\log\phi(a_t|s_t)$. 
% \vspace{-2.5mm}
\paragraph{DPO-Based Preference Estimation}
As provided by \citet{rafailov2024from}, the log ratio between a DPO-tuned LLM $\phi_\text{DPO}$ and a reference LLM $\phi_\text{ref}$ can serve as a token-level advantage signal for RL training. This provides an implicit representation of the advantage at time step $t$:
\begin{equation}
    \begin{aligned}
        \log\frac{\phi_\text{DPO}(a_t|s_t)}{\phi_\text{ref}(a_t|s_t)}&=r(s_t,a_t)+V(s_{t+1})-V(s_t).
    \end{aligned}
\end{equation}
The advantage estimator is defined as: $A(s_t,a_t):=\log\frac{\phi_\text{DPO}(a_t|s_t)}{\phi_\text{ref}(a_t|s_t)}$. 
Due to its natural alignment with the Bradley-Terry preference model and its widespread use in previous literature \citep{zhong2024dpo,chen2025discriminative}, we adopt this dense estimator as the default for our primary experiments.
% \vspace{-2mm}
\subsection{Gradient Analysis of LCO}
% \vspace{-1mm}
\label{sec:lco-gradient-analysis}
To characterize the optimization landscapes of different training objectives for LLMs, we first establish the relationship between the parameter-space gradient and the functional change in the logit space. Consider a loss function $\mathcal{L}$ with a finite lower bound $\mathcal{L}^*>-\infty$. Let $\boldsymbol{z}_\theta\in\mathbb{R}^{|\mathcal{V}|}$ denote the logit vector produced by an LLM with parameters $\theta$. We define $\Theta^* = \{\theta^*: \mathcal{L}(\boldsymbol{z}_{\theta^*}) - \mathcal{L}^* < \epsilon\}$ as the set of near-optimal parameters for a given tolerance $\epsilon > 0$.
% To theoretically characterize the optimization landscape of different objectives in the context of LLM training, we first establish the relationship between the parameter-space gradient and the functional change in the logits space. We consider a loss function $\mathcal{L}$ has a finite lower bound $\mathcal{L}^*>-\infty$. Let $\boldsymbol{z}_\theta$ denote the logits produced by a LLM with parameters $\theta$. We define $\Theta^*=\{\theta:\mathcal{L}(\boldsymbol{z}_\theta)-\mathcal{L}^*<\epsilon\}$ as the set of near-optimal parameters for a given tolerance $\epsilon>0$.
\begin{assumption}
    \label{ass:second-order-assumption}
    For any $\theta^*\in\Theta^*$ in the local neighborhood of $\theta$, we assume that the network operates in a regime where the parameter displacement $\Delta\theta=\theta^*-\theta$ is sufficiently small, so that the first-order Taylor expansion:
    % \vspace{-1.5mm}
    \begin{equation}
        \boldsymbol{z}_{\theta^*}\approx\boldsymbol{z}_\theta+\nabla_\theta\boldsymbol{z}_\theta^\top\Delta\theta
    \end{equation}
    % \vspace{-2mm}
    holds with negligible higher-order residual terms.
\end{assumption}
This assumption aligns with Neural Tangent Kernel theory \citep{jacot2018neural} and is supported by other literature: in over-parameterized networks (especially LLMs), individual parameters change negligibly while their collective contribution to the final output is significant \citep{lee2019wide}. We now introduce the \emph{logits convexity}, which serves as an indicator for well-behaved optimization.

\begin{definition}[Logits Convexity]
    \label{def:logits-convexity}
    Let $\mathcal{L}$ be a twice-differentiable loss function that takes logits $\boldsymbol{z}_\theta$ parameterized by $\theta$ as input. The $\mathcal{L}$ is called logits convex if its Hessian matrix with respect to all $\boldsymbol{z}_\theta$ is positive semi-definite.
\end{definition}
% \vspace{-0.5mm}
The following proposition establishes why this property is desirable for favorable parameter updates:
\begin{proposition}[Gradient Directionality]
    \label{prop:gradient-directionality}
    Under Assumption \ref{ass:second-order-assumption}, if $\mathcal{L}$ is logits convex, then the parameter-space gradient and logit-space gradient satisfy: 
    % \vspace{-1mm}
    \begin{equation}
        \langle \nabla_\theta\mathcal{L}, \theta - \theta^* \rangle \approx 
        \langle \nabla_{\boldsymbol{z}_\theta}\mathcal{L}, \boldsymbol{z}_\theta - \boldsymbol{z}_{\theta^*} \rangle
        \geq 0.
    \end{equation}
\end{proposition}
% \vspace{-2mm}
Derivation can be found in Appendix \ref{sec:proof-of-gradient-directionality}. Proposition \ref{prop:gradient-directionality} implies that the negative gradient of the loss $\mathcal{L}$ with logits convexity consistently points toward the near-optimal parameter $\theta^*$ in a local sense. This provides a theoretical guarantee that gradient descent on those objectives is not misled by spurious stationary points that may arise from non-convexity in the parameter landscape.
Below, we analyze the logits convexity of SFT, PPO, and LCO objectives:
\begin{lemma}
    \label{lemma:ppo-logits-convexity}
    The SFT loss function $\mathcal{L}_\text{SFT}$, as defined in Equation \eqref{eq:sft-objective}, is logits convex at each time step. In contrast, the PPO loss function $\mathcal{L}_\text{PPO}$, defined in Equation \eqref{eq:ppo-objective}, is not logits convex at any time step. (Proof. See Appendix \ref{sec:proof-of-sft-convexity}.)
\end{lemma}
These findings elucidate the disparity in gradient behavior observed between SFT and PPO. The non-convexity of the PPO objective in logit space can lead to a violation of gradient directionality, resulting in the erratic gradient spikes and explosions frequently encountered in empirical settings. In contrast, LCO restores this desirable optimization property: 
\begin{lemma}
    \label{lemma:lco-logits-convexity}
    The LCO loss functions $\mathcal{L}_\text{LCO-MSE}$, $\mathcal{L}_\text{LCO-LCH}$, and $\mathcal{L}_\text{LCO-KLD}$, defined in Equations \eqref{eq:lco-mse-objective} to \eqref{eq:lco-kld-objective}, are all logits convex at each time step. Specifically, $\mathcal{L}_\text{LCO-MSE}$ is strongly convex, while $\mathcal{L}_\text{LCO-LCH}$ is locally strongly convex on any bounded set of logits. (Proof. See Appendix \ref{sec:proof-of-lco-mse-convexity}.)
\end{lemma}
While standard PPO attempts to mitigate this instability through heuristic constraints like clipping or trust regions, the LCO objectives offer a more fundamental solution grounded in Proposition \ref{prop:gradient-directionality}. By ensuring convexity in the logit space, the parameter-space optimization naturally inherits a well-behaved landscape, ensuring that the model updates remain stable throughout the training.

\begin{figure*}[t]
    \centering
    % \vspace{-12mm}
    \includegraphics[width=0.99\linewidth]{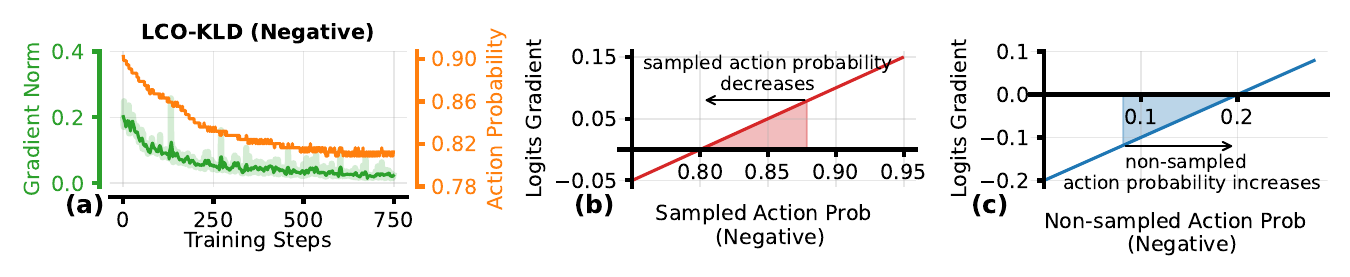}
    \vspace{-2mm}
    \caption{Training dynamics of LCO-KLD for actions with negative advantages. Unlike training dynamics of PPO, LCO‑KLD exhibits stable gradient updates that scale with the remaining error and diminish progressively as the model converges.}
    \label{fig:introduction-convex}
    % \vspace{-5mm}
\end{figure*}

% Beyond demonstrating that logits convexity ensures favorable gradient directionality, we further show that the gradient norm upper bounds of the LCO objectives decrease monotonically as the loss diminishes:
Complementing the guarantee of favorable gradient directionality, we further establish that the gradient norms of LCO objectives are bounded by monotonic functions of the loss. This property ensures gradient magnitude naturally dissipates as the model approaches the optimal target:
\begin{proposition}[Gradient Norm Upper Bounds]
    \label{prop:gradient-norm-upper-bound}
    We consider LCO objectives defined in Equations \eqref{eq:lco-mse-objective} to \eqref{eq:lco-kld-objective} at time step $t$. Let $\sigma_{\max}$ denote the maximum singular value of  $\nabla_\theta\boldsymbol{z}_\theta$. The gradient norm upper bound of $\mathcal{L}_\text{LCO-MSE}$ is:
    % \vspace{-1mm}
    \begin{equation}
        \|\nabla_\theta\mathcal{L}_\text{LCO-MSE}\|\leq\frac{2}{|\mathcal{V}|}\sigma_\text{max}\sqrt{|\mathcal{V}|\mathcal{L}_\text{LCO-MSE}}.
    \end{equation}
    % \vspace{-1mm}
    The gradient norm upper bound of $\mathcal{L}_\text{LCO-LCH}$ is:
    % \vspace{-1mm}
    \begin{equation}
        \|\nabla_\theta\mathcal{L}_\text{LCO-LCH}\|\leq\frac{1}{|\mathcal{V}|}\sigma_\text{max}\sqrt{|\mathcal{V}|(1-\exp(-2\mathcal{L}_\text{LCO-LCH}))}.
    \end{equation}
    % \vspace{-1mm}
    The gradient norm upper bound of $\mathcal{L}_\text{LCO-KLD}$ is:
    % \vspace{-1mm}
    \begin{equation}
        \|\nabla_\theta\mathcal{L}_\text{LCO-KLD}\|\leq\sigma_\text{max}\sqrt{2\mathcal{L}_\text{LCO-KLD}}.
    \end{equation}
\end{proposition}
% \vspace{-2mm}
\begin{proof}
    See Appendix \ref{sec:proof-of-grdient-upper-bound}.
\end{proof}
% \vspace{-2mm}
Proposition \ref{prop:gradient-norm-upper-bound} elucidates an inherent self-stabilizing mechanism within LCO: the update magnitude is coupled to the remaining error. Unlike traditional RL objectives that often require heuristic gradient clipping to prevent divergence, LCO’s gradient pressure subsides gracefully as the policy aligns with the target. This behavior effectively precludes the gradient spikes that typically destabilize RL training. As visualized in Figures \ref{fig:introduction-convex} and \ref{fig:introduction-convex-positive}, the optimization landscape remains smooth, with empirical gradient trajectories confirming a controlled decay toward the optimum.
Building upon these gradient properties, we formalize the convergence behavior of the LCO framework:
\begin{proposition}[LCO Convergence Guarantees]
    \label{prop:lco-convergence}
    Under Assumption \ref{ass:second-order-assumption}, for a sufficiently small step size, the LCO objectives $\mathcal{L}_\text{LCO-MSE}$, $\mathcal{L}_\text{LCO-LCH}$, and $\mathcal{L}_\text{LCO-KLD}$ converge at a sublinear rate of at least $\mathcal{O}(1/K)$, where $K$ is the number of iterations. (Proof. See Appendix \ref{sec:lco-convergence-guarantees}.)
\end{proposition}
% \begin{proof}
%     See Appendix \ref{sec:lco-convergence-guarantees}.
% \end{proof}
By providing both a bounded gradient norm and a guaranteed convergence rate, LCO offers a robust theoretical foundation for the policy optimization, ensuring efficient and stable training process.
% that the training process is both efficient and numerically stable.

% \begin{proposition}[LCO Convergence Guarantees]
%     Under Assumption \ref{ass:second-order-assumption}, the LCO loss functions $\mathcal{L}_\text{LCO-MSE}$, $\mathcal{L}_\text{LCO-LCH}$, and $\mathcal{L}_\text{LCO-KLD}$, defined in Equations \eqref{eq:lco-mse-objective} to \eqref{eq:lco-kld-objective}, at least converges at a sublinear rate.
% \end{proposition}
% \begin{proof}
%     See Appendix \ref{sec:lco-convergence-guarantees}
% \end{proof}
% This behavior ensures LCO exhibits a stable and well-behaved convergence in reinforcement learning for LLMs.

% confirming our theoretical finding that as the parameterized policy aligns with the optimal target, the optimization pressure gracefully subsides, leading to a stable and well-behaved convergence trajectory.

% Proposition \ref{prop:gradient-norm-upper-bound} demonstrates that LCO objectives produce stable gradient updates that diminish progressively as training converges, thereby preventing sudden gradient spikes that could otherwise undermine the stability of RL training. Figure \ref{fig:introduction-convex} visualizes the gradient dynamics of $\mathcal{L}_\text{LCO-KLD}$. As training converges, the magnitude of the parameter gradients diminishes, indicating stable gradient dynamics.

\section{Experimental Setup}
%\paragraph{Evaluation Tasks} Our evaluation tasks include AMC23, MATH500  \citep{lightman2023let} and MinervaMath \citep{lewkowycz2022solving} for math reasoning, QA-Feedback \citep{wu2023fine} for machine reading comprehension, and AlpacaEval 2.0 \citep{dubois2023alpacafarm} for instruction following. For math reasoning tasks, we report Pass@1 performance. For QA-Feedback, we follow \citet{wu2023fine} to utilize relevance, factuality, and completeness RMs for evaluation, and then report the test set reward of policy on these RMs. For AlpacaEval, we assess the quality of policy response in comparison to GPT-4-1106-Preview, using win rate and length-controlled win rate (LC Win Rate) as metrics, with Azure GPT-4o-2024-11-20 serving as the judge. 
% \vspace{-2.5mm}
\paragraph{Evaluation Tasks} We evaluate the LCO framework across three primary domains: mathematical reasoning (AMC23 \citep{maa2023amc}, MATH500 \citep{lightman2023let}, and MinervaMath \citep{lewkowycz2022solving}), machine reading comprehension (QA-Feedback \citep{wu2023fine}), and instruction following (AlpacaEval 2.0 \citep{dubois2023alpacafarm}). For math reasoning, we report the Pass@1 performance. For QA-Feedback, following \citet{wu2023fine}, we utilize reward models for relevance, factuality, and completeness, reporting the policy's reward scores on the test set. For AlpacaEval 2.0, we assess response quality relative to GPT-4-1106-Preview using win rate and length-controlled (LC) win rate, with Azure GPT-4o-2024-11-20 serving as judge.

%\paragraph{Training Data} We first introduce the datasets utilized for RL training. For math reasoning, we utilize the training instructions of MATH. For QA-Feedback, we adopt the original training instructions for RL traiing. For AlpacaEval 2.0, we utilize the instructions from UltraFeedback \citep{cui2023ultrafeedback} for RL training. 
%Next, we introduce the datasets used for reward models (RMs) training. For math reasoning, we utilize original training set of MATH. For each instruction, we leverage Llama-3-8B/70B and Qwen-2.5-32B/72B for sampling. Responses are evaluated against the labels, the correct responses are selected as chosen samples, while incorrect ones are marked as rejected. After filtering and deduplication, we obtain 30,183 pairwise samples. For QA-Feedback, we utilize original preference data from their training set, totaling 17,835 pairs. For AlpacaEval 2.0, we utilize UltraFeedback as preference data. 

% \vspace{-2.5mm}
\paragraph{Training Data} We first detail the datasets employed for RL training. For mathematical reasoning, we utilize the original training instructions from the MATH dataset, with 500 samples randomly held out as an evaluation set. For QA-Feedback, we adopt its original training instructions. For AlpacaEval 2.0, we utilize instructions from UltraFeedback \citep{cui2023ultrafeedback}. Regarding RM training, the datasets are curated as follows. For mathematical reasoning, we sample responses from Llama-3 (8B/70B) and Qwen-2.5 (32B/72B) using MATH training instructions. Responses are labeled as chosen or rejected based on their correctness against ground truth, yielding 30,183 pairs after filtering. For QA-Feedback, we use its original preference data, totaling 17,835 pairs. For AlpacaEval 2.0, UltraFeedback serves as the preference dataset.
% For SFT warm-up, 1,000 chosen responses were randomly selected from each of the above datasets and train 1 epoch for initial instruction following capabilities.
For SFT warm-up, we select 1k chosen responses from each of the above datasets, and train the policy to ensure initial instruction-following capabilities.

%\paragraph{Baselines} We compare our approach against several baseline methods, including REINFORCE \citep{williams1992simple}, PPO \citep{schulman2017proximal}, GRPO \citep{shao2024deepseekmath}, DAPO \citep{yu2025dapo} and GSPO \cite{zheng2025group}. For a more extensive evaluation, we also include on-policy distillation methods such as MiniLLM \citep{gu2023minillm} and GKD \citep{agarwal2024policy}. To ensure consistency, we follow the recommended configuration from their original papers as well as the default settings from widely used repositories \footnote{\url{https://github.com/huggingface/trl}}.

% \vspace{-2.5mm}
\paragraph{Baselines} We evaluate our approach against several representative baselines, including REINFORCE \citep{williams1992simple} and PPO \citep{schulman2017proximal}, as well as GRPO \citep{shao2024deepseekmath} and its stability-enhanced variants, namely DAPO \citep{yu2025dapo} and GSPO \citep{zheng2025group}. To broaden the scope of evaluation, we also include on-policy distillation methods such as MiniLLM \citep{gu2023minillm} and GKD \citep{agarwal2024policy}. 
Furthermore, the RM and $\pi^*$ defined in Equation \eqref{eq:trpo-optimal-policy} are also included.
For a fair comparison, all methods are configured following their original reports and the default settings in the TRL library\footnote{\url{https://github.com/huggingface/trl}}.

%\paragraph{Reward Model Training} Following \citet{zhong2024dpo}, we train a LLM using DPO as RM to provide reward signals during RL training. To ensure a fair comparison, for on-policy distillation methods, a same RM is using as the teacher model. Our experiments cover various RM backbones, including Qwen-2.5-7B, Qwen-3-8B, Llama-3.1-8B and Mistral-3-8B. To mitigate overfitting, training is restricted to a single epoch with a learning rate 1e-6. 

\vspace{-2mm}
\paragraph{Training Details} Following \citet{zhong2024dpo}, we train an LLM via DPO to serve as the RM. In our main experiments, RMs are built on various backbones, including Qwen-2.5-7B, Qwen-3-8B, Llama-3.1-8B, and Mistral-3-8B. RM training is conducted for 1 epoch with a learning rate of 1e-6 to mitigate overfitting risks.
%For RL Training, our experiments cover various policy backbones, including Qwen-2.5-3B, Qwen-3-4B, Llama-3.2-3B and Mistral-3-3B. All policy models use RMs with the same backbone (i.e., Qwen-2.5-3B as the policy and Qwen-2.5-7B as the RM). We set the max sequence length to 2,048, with 4 responses generated per instruction for GRPO, DAPO, and GSPO, while one response generated per instruction for others. A sampling temperature of 0.6 and a top-$p$ value of 0.95 are consistently applied across all policies to control the diversity and quality of generated responses, as recommended by \citet{guo2025deepseek}. For LCO objectives, we set $\beta=1.0$ as a default value across all experiments. We set the learning rate to 5e-6 to ensure effective training for LCO.
Regarding RL training, we evaluate policy backbones including Qwen-2.5-3B, Qwen-3-4B, Llama-3.2-3B, and Mistral-3-3B. Each policy is consistently paired with an RM from its corresponding model series (e.g., a Qwen-2.5-3B policy with a Qwen-2.5-7B RM). We set the sequence length at 2,048. For GRPO, DAPO, and GSPO, we sample four responses per instruction, while other methods use one. 
% Consistent with \citet{guo2025deepseek}, 
We apply a sampling temperature of 0.6 and top-$p$ of 0.95. For the LCO objectives, we adopt a default $\beta=1.0$ and a learning rate of 5e-6. Evaluation results are reported as the mean and standard deviation across four random seeds.

\begin{wraptable}{r}{0.60\textwidth}
    \centering
    \renewcommand{\arraystretch}{1.0}
    \vspace{-4mm}
    \caption{Main results of Qwen-2.5-3B and Qwen-3-4B on math reasoning tasks with standard error. Best performances are shown in \textbf{bold}, while suboptimal ones are \underline{underlined}.}
    \vspace{-2mm}
    \label{tab:main-results-math}
    \begin{adjustbox}{width=0.99\linewidth}
        \begin{tabular}{l|ccc|ccc}
        \toprule
        \multirow{3}*{\textbf{Methods}} & \multicolumn{3}{c|}{\textbf{Qwen-2.5-3B}} & \multicolumn{3}{c}{\textbf{Qwen-3-4B}} \\
         & \textbf{MATH500} & \textbf{AMC23} & \textbf{MinervaMath} & \textbf{MATH500} & \textbf{AMC23} & \textbf{MinervaMath} \\
         & Pass@1 & Pass@1 & Pass@1 & Pass@1 & Pass@1 & Pass@1 \\
        \hline
        \hline
        $\phi_\text{DPO}$ & 58.30$\pm$1.83 & 35.70$\pm$2.37 & 16.17$\pm$1.56 & 69.20$\pm$2.31 & 58.20$\pm$3.18 & 16.54$\pm$1.34 \\
        $\pi^*$ & 51.30$\pm$1.53 & 28.40$\pm$2.49 & 15.18$\pm$0.81 & 61.35$\pm$1.67 & 47.80$\pm$3.10 & 14.39$\pm$0.71 \\
        \hline
        \hline
        SFT & 41.60$\pm$0.72 & 27.50$\pm$1.44 & 10.29$\pm$0.71 & 58.80$\pm$0.50 & 45.65$\pm$1.81 & 12.53$\pm$0.34 \\
        \hline
        REINFORCE & 54.60$\pm$0.53 & 40.15$\pm$1.25 & 12.35$\pm$0.93 & 64.80$\pm$0.68 & 46.65$\pm$2.43 & 19.48$\pm$0.41 \\
        PPO & 55.40$\pm$1.02 & 39.20$\pm$2.01 & 13.97$\pm$0.81 & 67.80$\pm$0.83 & 47.50$\pm$1.25 & 20.95$\pm$0.28 \\
        GRPO & 57.79$\pm$0.34 & \underline{41.75$\pm$1.25} & 14.89$\pm$0.87 & 67.60$\pm$1.23 & 51.95$\pm$1.92 & 21.75$\pm$0.19 \\
        DAPO & 57.05$\pm$0.31 & 40.20$\pm$1.46 & 13.23$\pm$0.78 & 67.80$\pm$0.78 & \underline{53.15$\pm$2.45} & 19.75$\pm$0.34 \\
        GSPO & 56.80$\pm$1.38 & 38.75$\pm$2.01 & 14.87$\pm$0.71 & 66.40$\pm$1.64 & 51.20$\pm$2.07 & 18.38$\pm$0.45 \\
        \hline
        \rowcolor{mylightblue} LCO-MSE & 59.20$\pm$1.75 & 40.20$\pm$1.58 & 15.12$\pm$0.43 & 67.20$\pm$1.35 & 51.30$\pm$1.23 & 21.37$\pm$0.36  \\
        \rowcolor{mylightblue} LCO-LCH & \textbf{61.40$\pm$2.01} & 39.50$\pm$2.37 & \underline{15.38$\pm$0.16} & \underline{69.40$\pm$1.98} & 52.70$\pm$2.45 & \textbf{24.26$\pm$0.21} \\
        \rowcolor{mylightblue} LCO-KLD & \underline{60.80$\pm$0.51} & \textbf{42.65$\pm$1.63} & \textbf{16.71$\pm$0.97} & \textbf{73.20$\pm$1.34} & \textbf{55.50$\pm$1.88} & \underline{23.95$\pm$0.31} \\
        \bottomrule
        \end{tabular}
    \end{adjustbox}
    \vspace{-2mm}
\end{wraptable}

% \vspace{-2mm}
\section{Results and Analysis}
\vspace{-1mm}
%For all main results, we report the average and standard deviation across 4 random seeds.
\subsection{Main Results}
\vspace{-2mm}
\paragraph{Math Reasoning} 
Table \ref{tab:main-results-math} summarizes the performance on mathematical reasoning tasks using Qwen-2.5-3B and Qwen-3-4B as base models. LCO objectives surpass competitive RL baselines 
% (e.g., REINFORCE, PPO, GRPO, DAPO, and GSPO) 
across most benchmarks such as MATH500 and MinervaMath. Notably, when built upon Qwen-2.5-3B, LCO-LCH yields the top Pass@1 score of 61.40\% on MATH500, while LCO-KLD excels on AMC23 (42.65\%) and MinervaMath (16.71\%). Similarly, using Qwen-3-4B, LCO-KLD achieves state-of-the-art performance on MATH500 (73.20\%) and AMC (55.50\%). 
Analysis of LCO variants reveals that while LCO-MSE is sensitive to advantage noise, LCO-LCH mitigates outlier influence through log-cosh penalty, contributing to its better performance on Qwen-3-4B. Meanwhile, LCO-KLD demonstrates the highest robustness by leveraging probabilistic consistency across the action space.
Furthermore, we observe that LCO surpasses the performance of $\pi^*$. Unlike a fixed target derived from the initial $\pi_{\text{old}}$ and advantage, LCO facilitates a dynamic optimization process that continuously aligns with evolving targets. Notably, most LCO variants even exceed the performance of the RM $\phi_\text{DPO}$ despite its larger parameter count. 
These results demonstrate the versatility of LCO in improving mathematical reasoning performance across diverse model backbones.

\begin{wraptable}{r}{0.60\textwidth}
    \centering
    \renewcommand{\arraystretch}{1.0}
    \vspace{-4mm}
    \caption{Results for QA-Feedback in relevance (Rel.), factuality (Fact.), and completeness (Comp.) with standard error.}
    \vspace{-2mm}
    \label{tab:main-results-qa-feedback}
    \begin{adjustbox}{width=0.99\linewidth}
        \begin{tabular}{l|cccc|cccc}
        \toprule
        \multirow{3}*{\textbf{Methods}} & \multicolumn{4}{c|}{\textbf{Llama-3.2-3B}} & \multicolumn{4}{c}{\textbf{Mistral-3-3B}} \\
        & \textbf{Rel.} & \textbf{Fact.} & \textbf{Comp.} & \multirow{2}*{\textbf{Avg$(\uparrow)$}} & \textbf{Rel.} & \textbf{Fact.} & \textbf{Comp.} & \multirow{2}*{\textbf{Avg$(\uparrow)$}} \\
        & $R_1(\uparrow)$ & $R_2(\uparrow)$ & $R_3(\uparrow)$ & & $R_1(\uparrow)$ & $R_2(\uparrow)$ & $R_3(\uparrow)$ & \\
        \hline
        \hline
        $\phi_\text{DPO}$ & 0.369 & 0.737 & 0.549 & 0.552$\pm$0.009 & 0.412 & 0.658 & 0.559 & 0.556$\pm$0.013 \\
        $\pi^*$ & 0.321 & 0.691 & 0.514 & 0.509$\pm$0.010 & 0.416 & 0.597 & 0.528 & 0.514$\pm$0.003 \\
        \hline
        \hline
        SFT & 0.352 & 0.597 & 0.488 & 0.479$\pm$0.011 & 0.427 & 0.589 & 0.513 & 0.510$\pm$0.013 \\
        \hline
        REINFORCE & 0.280 & 0.632 & 0.506 & 0.473$\pm$0.005 & 0.416 & 0.629 & 0.536 & 0.527$\pm$0.011 \\ 
        PPO & 0.374 & 0.627 & 0.507 & 0.503$\pm$0.007 & \underline{0.431} & 0.614 & 0.531 & 0.525$\pm$0.002 \\
        GRPO & 0.342 & 0.708 & 0.516 & 0.522$\pm$0.009 & 0.410 & 0.622 & 0.529 & 0.520$\pm$0.005 \\
        DAPO & 0.337 & 0.721 & 0.516 & 0.525$\pm$0.003 & 0.412 & 0.625 & 0.528 & 0.522$\pm$0.003 \\
        GSPO & 0.332 & 0.724 & 0.516 & 0.524$\pm$0.006 & 0.392 & 0.639 & 0.522 & 0.518$\pm$0.007 \\
        \hline
        \rowcolor{mylightblue} LCO-MSE & 0.378 & 0.793 & 0.539 & 0.570$\pm$0.005 & 0.395 & \underline{0.669} & 0.552 & \underline{0.539$\pm$0.003} \\
        \rowcolor{mylightblue} LCO-LCH & \underline{0.383} & \underline{0.814} & \underline{0.543} & \underline{0.580$\pm$0.002} & 0.396 & 0.639 & \underline{0.577} & 0.537$\pm$0.004 \\
        \rowcolor{mylightblue} LCO-KLD & \textbf{0.392} & \textbf{0.817} & \textbf{0.612} & \textbf{0.607$\pm$0.003} & \textbf{0.449} & \textbf{0.673} & \textbf{0.620} & \textbf{0.581$\pm$0.006} \\ 
        \bottomrule
        \end{tabular}
    \end{adjustbox}
    \vspace{-2mm}
\end{wraptable}
% \vspace{-2mm}
\paragraph{Machine Reading Comprehension} 
We report the QA-Feedback results in Table \ref{tab:main-results-qa-feedback}, employing Llama-3.2-3B and Mistral-3-3B as backbones. Performance is measured by three oracle RMs: Relevance, Factuality, and Completeness. LCO variants consistently demonstrate superiority over all baselines. Specifically, LCO-KLD achieves the highest average rewards of 0.607 and 0.581 on Llama-3.2-3B and Mistral-3-3B, respectively, significantly surpassing PPO (0.503 and 0.525) and RM $\phi_\text{DPO}$ (0.552 and 0.556). Compared to SFT and $\pi^*$, LCO variants exhibit consistent gains, with LCO-KLD achieving 0.607 on Llama-3.2-3B and 0.581 on Mistral-3-3B. Moreover, all LCO variants surpass the competitive DAPO and GSPO baselines, especially in factuality (up to 0.817) and completeness (up to 0.620). These results highlight the effectiveness of LCO objectives in aligning policies to generate high-quality responses.

\begin{wraptable}{r}{0.60\textwidth}
    \centering
    \renewcommand{\arraystretch}{1.0}
    \vspace{-4mm}
    \caption{Results for AlpacaEval 2.0 with standard error.}
    \vspace{-2mm}
    \label{tab:main-results-alpaca-eval}
    \begin{adjustbox}{width=0.99\linewidth}
        \begin{tabular}{l|cc|cc|cc}
        \toprule
        \multirow{2}*{\textbf{Methods}} & \multicolumn{2}{c|}{\textbf{Qwen-3-4B}} & \multicolumn{2}{c|}{\textbf{Llama-3.2-3B}} & \multicolumn{2}{c}{\textbf{Mistral-3-3B}} \\
        
        & \textbf{WR (\%)} & \textbf{LC WR (\%)} & \textbf{WR (\%)} & \textbf{LC WR (\%)} & \textbf{WR (\%)} & \textbf{LC WR (\%)} \\
        \hline
        \hline
        $\phi_\text{DPO}$ & 32.09$\pm$1.67 & 28.56$\pm$0.10 & 19.74$\pm$1.33 & 21.34$\pm$0.05 & 26.78$\pm$0.72 & 27.32$\pm$0.17 \\
        $\pi^*$ & 25.12$\pm$0.91 & 24.31$\pm$0.09 & 18.31$\pm$0.72 & 21.18$\pm$0.16 & 23.58$\pm$0.67 & 24.17$\pm$0.11 \\
        \hline
        \hline
        SFT & 23.99$\pm$1.49 & 25.95$\pm$0.01 & 17.73$\pm$1.25 & 18.04$\pm$0.12 & 19.39$\pm$1.31 & 22.29$\pm$0.07 \\
        \hline
        REINFORCE & 25.14$\pm$1.37 & 27.18$\pm$0.02 & 21.85$\pm$1.32 & 22.68$\pm$0.56 & 21.57$\pm$1.82 & 23.14$\pm$0.04 \\
        PPO & 26.79$\pm$1.45 & 27.20$\pm$0.07 & 23.74$\pm$1.21 & 25.76$\pm$0.71 & 25.81$\pm$1.45 & 26.41$\pm$0.05 \\
        GRPO & 25.37$\pm$1.65 & 26.38$\pm$0.21 & 23.72$\pm$1.32 & 24.71$\pm$0.43 & 25.70$\pm$1.23 & \textbf{27.81$\pm$0.09} \\
        DAPO & 24.67$\pm$1.37 & 25.18$\pm$0.18 & 22.07$\pm$1.22 & 24.81$\pm$0.34 & 24.91$\pm$1.19 & 26.71$\pm$0.08 \\
        GSPO & 25.67$\pm$1.12 & 26.17$\pm$0.11 & 24.17$\pm$1.98 & 25.13$\pm$0.32 & 23.85$\pm$1.71 & 27.31$\pm$0.13 \\
        \hline
        \rowcolor{mylightblue} LCO-MSE & 27.01$\pm$1.62 & 29.37$\pm$0.03 & 23.87$\pm$1.32 & \underline{25.86$\pm$0.32} & 25.31$\pm$1.01 & 26.87$\pm$0.04 \\
        \rowcolor{mylightblue} LCO-LCH & \underline{27.65$\pm$1.43} & \underline{31.91$\pm$0.02} & \textbf{25.37$\pm$1.21} & \textbf{26.82$\pm$0.09} & \underline{26.71$\pm$1.32} & \underline{27.41$\pm$0.02} \\
        \rowcolor{mylightblue} LCO-KLD & \textbf{29.05$\pm$1.49} & \textbf{32.93$\pm$0.03} & \underline{24.38$\pm$1.16} & 25.55$\pm$0.12 & \textbf{26.81$\pm$1.23} & 27.10$\pm$0.03 \\
        \bottomrule
        \end{tabular}
    \end{adjustbox}
    % \vspace{-2mm}
\end{wraptable}

% \vspace{-2mm}
% \paragraph{Instruction Following} Table \ref{tab:main-results-alpaca-eval} illustrates that LCO objectives yield substantial gains in AlpacaEval 2.0. Specifically, LCO-KLD with a Qwen-3-4B backbone attains 29.05\% WR and 32.93\% LC WR, surpassing PPO (26.79\%/27.20\%), GRPO (25.37\%/26.38\%), and GSPO (25.67\%/26.17\%) by a clear margin. With Llama‑3.2‑3B, LCO‑LCH and LCO-KLD achieve their highest comparative win rates at 25.37\% and 24.38\%, while under Mistral‑3‑3B, LCO‑KLD achieves the highest win rate of 26.81\%. These improvements across diverse model families underscore the robustness of LCO objectives in instruction-following tasks.

\paragraph{Instruction Following} Table \ref{tab:main-results-alpaca-eval} illustrates that LCO objectives yield substantial gains in the AlpacaEval 2.0 benchmark. Specifically, LCO-KLD with a Qwen-3-4B backbone attains 29.05\% WR and 32.93\% LC WR, surpassing PPO (26.79\%/27.20\%), GRPO (25.37\%/26.38\%), and GSPO (25.67\%/26.17\%) by a clear margin. This represents an absolute improvement of over 6 percentage points in LC WR compared to the strongest baseline, demonstrating the effectiveness of logit-space convexity in aligning model outputs with human preferences. With Llama‑3.2‑3B, LCO‑LCH and LCO-KLD achieve their highest comparative win rates at 25.37\% and 24.38\%, respectively, while under Mistral‑3‑3B, LCO‑KLD achieves the highest win rate of 26.81\%. Notably, LCO-LCH consistently ranks among the top two methods across all three backbones, suggesting that this logit convexity variant offers a particularly robust trade-off between training stability and alignment performance. These improvements across diverse model families underscore the robustness of LCO objectives in instruction-following tasks, confirming that the proposed framework generalizes well beyond any single model configuration.

%As shown in Table \ref{tab:main-results-alpaca-eval}, LCO objectives yield strong performance gains on instruction-following tasks. Under the Qwen-3-4B backbone, LCO‑KLD attains win rates of 29.05\% (standard) and 32.93\% (LC) on AlpacaEval2.0, surpassing results from PPO (26.79\%/27.20\%), GRPO (25.37\%/26.38\%), and GSPO (25.67\%/26.17\%). With Llama‑3.2‑3B, LCO‑LCH reaches win rates of 25.37\% and 26.82\%, while under Mistral‑3‑3B, LCO‑KLD achieves the highest win rate of 26.81\%. These consistent improvements across multiple model families demonstrate the robustness of the LCO objectives in instruction-following settings.

% \vspace{-1mm}
\subsection{Ablation Study and Analysis}
% \vspace{-1mm}
%In this subsection, we conduct experiments to investigate the impact of different factors of LCO on RL training.
\begin{figure*}[b]
    \centering
    % \vspace{-6mm}
    \includegraphics[width=0.99\linewidth]{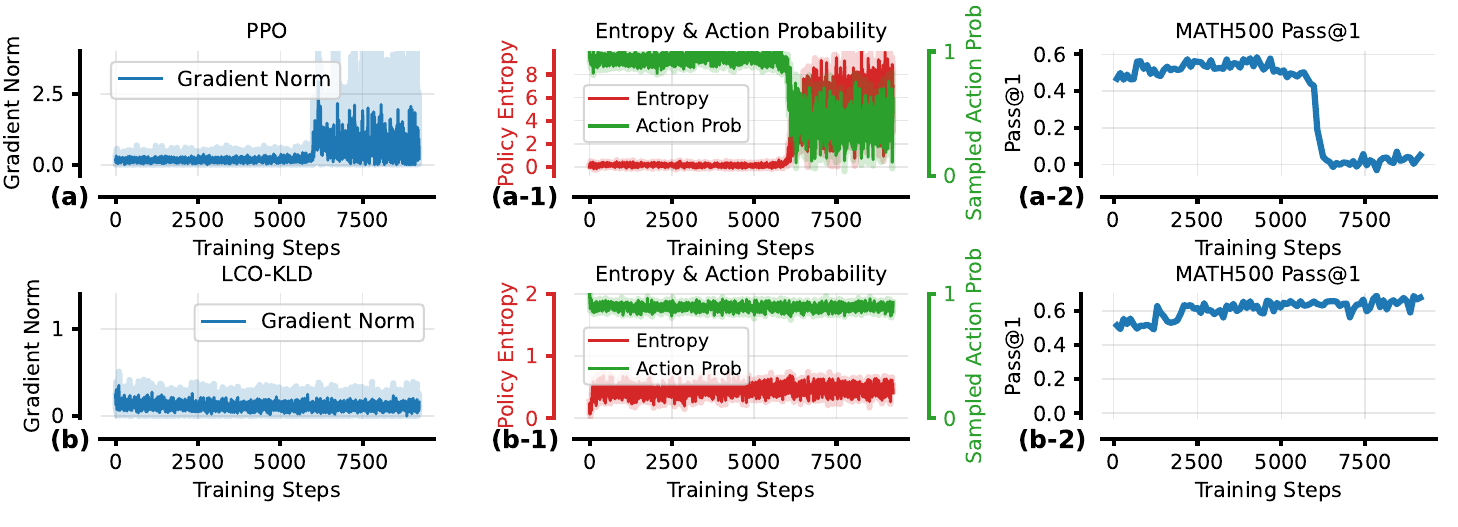}
    \vspace{-2mm}
    \caption{(a) and (b) show the gradient norm dynamics of PPO and LCO-KLD, respectively. (a-1) and (b-1) show the policy entropy and sampled action probability dynamics, while (a-2) and (b-2) present the corresponding Pass@1 performance on the MATH500 dataset.}
    \label{fig:training-dynamics-results}
    % \vspace{-3mm}
\end{figure*}
\paragraph{Training Dynamics Analysis}
To investigate how LCO stabilizes RL training process, we compare training dynamics of LCO-KLD and PPO in Figure \ref{fig:training-dynamics-results}. (1) \textit{Gradient norms}: the gradient norm dynamics are illustrated in Figure \ref{fig:training-dynamics-results}(a) and (b). As training progresses, PPO remains relatively stable during the early stages but begins to oscillate after approximately 6K steps. In contrast, the gradient norm of LCO-KLD remains stable throughout the entire training. (2) \textit{Entropy and action probabilities}: we analyze the average entropy of the action distribution and probabilities of sampled actions to evaluate the policy's exploration capability. As shown in Figure \ref{fig:training-dynamics-results}(a-1) and (b-1), PPO exhibits a sharp drop in sampled action probabilities and a surge in entropy during later stages of training. 
This phenomenon is strongly linked to oscillations in gradient norms and likely stems from excessive gradients caused by negative advantages.
However, LCO-KLD maintains stable entropy and action probabilities, preserving exploration while ensuring effective optimization. (3) \textit{Performance}: as shown in Figure \ref{fig:training-dynamics-results}(a-2) and (b-2), PPO suffers a performance decline in the later stages due to training collapse, whereas LCO-KLD achieves consistent improvement and ultimately surpasses PPO. 
% These results highlight that LCO-KLD not only boosts policy performance but also ensures training stability.

\begin{wrapfigure}{r}{0.6\textwidth}
    \centering
    % \vspace{-2mm}
    \includegraphics[width=0.99\linewidth]{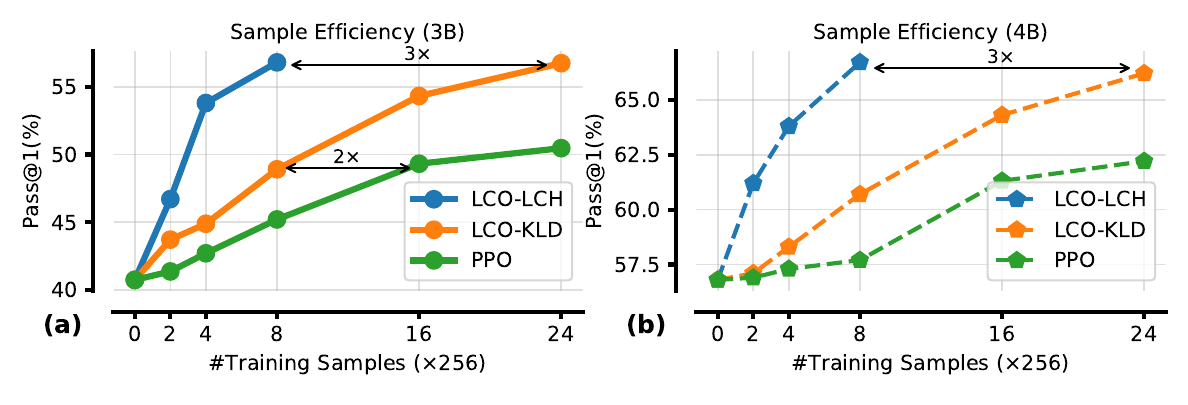}
    \vspace{-2mm}
    \caption{Training efficiency of LCO-LCH/KLD and PPO, evaluated by policy accuracy on MATH500 evaluation set.}
    \label{fig:training-efficiency}
    \vspace{-2mm}
\end{wrapfigure}

% \vspace{-2mm}
\paragraph{Training Efficiency} To assess the training efficiency of the LCO objectives, we compare the performance of policies based on Qwen‑3‑4B and Qwen‑2.5‑3B backbones, trained using LCO‑LCH, LCO‑KLD, and PPO.
% as the number of training samples increases. 
As illustrated in Figure \ref{fig:training-efficiency}(a), with the Qwen‑2.5‑3B backbone, LCO‑KLD achieves a Pass@1 comparable to PPO while requiring nearly half the training samples, owing to the logits convexity of LCO‑KLD, which accelerates convergence. LCO‑LCH further enhances sample efficiency. As depicted in Figure \ref{fig:training-efficiency}(a) and (b), LCO‑LCH achieves a Pass@1 similar to LCO‑KLD but is nearly three times more sample-efficient. This greater efficiency stems from the stronger convexity of LCO‑LCH in the logit space (Lemma \ref{lemma:lco-logits-convexity}), ensuring a faster convergence rate and speeds up training.

% This improvement is attributed to the stronger convexity of LCO‑LCH at the logit space, backed up by Proposition \ref{prop:lco-lch-convergence-rate} which demonstrate LCO-LCH has a linear convergence rate, which further accelerates convergence.

%Furthermore, LCO‑LCH improves sample efficiency even more: as shown in Figure \ref{fig:training-efficiency}(a) and (b), it reaches accuracy comparable to LCO‑KLD while being nearly three times as sample‑efficient. This gain can be attributed to the fact that LCO‑LCH is not only logits convex but strongly convex, further speeding up convergence.

\begin{wraptable}{r}{0.60\textwidth}
    \centering
    \renewcommand{\arraystretch}{1.1}
    \vspace{-4mm}
    \caption{Comparison with on-policy distillation methods.}
    \vspace{-2mm}
    \label{tab:main-results-math-distill}
    \begin{adjustbox}{width=0.99\linewidth}
        \begin{tabular}{l|ccc|ccc}
        \toprule
        \multirow{2}*{\textbf{Methods}} & \multicolumn{3}{c|}{\textbf{Qwen-2.5-3B}} & \multicolumn{3}{c}{\textbf{Qwen-3-4B}} \\
         & \textbf{MATH500} & \textbf{AMC23} & \textbf{MinervaMath} & \textbf{MATH500} & \textbf{AMC23} & \textbf{MinervaMath} \\
         % & Pass@1 & Pass@1 & Pass@1 & Pass@1 & Pass@1 & Pass@1 \\
        \hline
        \hline
        MiniLLM & 55.20$\pm$0.84 & \underline{40.85$\pm$1.24} & 13.60$\pm$0.45 & 66.30$\pm$1.93 & \underline{53.80$\pm$1.23} & 20.34$\pm$0.56 \\
        GKD & \underline{56.40$\pm$1.12} & 34.95$\pm$1.45 & \underline{15.91$\pm$0.89} & \underline{67.70$\pm$1.07} & 50.35$\pm$1.08 & \underline{21.69$\pm$0.54} \\
        \hline
        \rowcolor{mylightblue} LCO-KLD & \textbf{60.80$\pm$0.51} & \textbf{42.65$\pm$1.63} & \textbf{16.71$\pm$0.97} & \textbf{73.20$\pm$1.34} & \textbf{55.50$\pm$1.88} & \textbf{23.95$\pm$0.31} \\
        \bottomrule
        \end{tabular}
    \end{adjustbox}
    \vspace{-2mm}
\end{wraptable}

% \vspace{-2mm}
\paragraph{Compare to On-Policy Distillation} 
% To better understand the effectiveness of the LCO‑KLD objective, 
We further compare the LCO-KLD objective against two on‑policy distillation methods, GKD and MiniLLM. For a fair comparison, we use the same RMs ($\phi_\text{DPO}$) as teachers in these baselines, consistent with LCO-KLD. As presented in Table \ref{tab:main-results-math-distill}, LCO-KLD demonstrates competitive performance across three benchmarks. 
While GKD and MiniLLM primarily focus on mimicking the teacher distribution, LCO-KLD incorporates advantage information derived from original RL objective to guide the policy alignment. This performance gap underscores that LCO-KLD benefits from the dynamic nature of RL, which allow policy to potentially surpass the ``performance ceiling" often associated with traditional distillation.

% While GKD and MiniLLM are limited by the performance of their teacher models (see $\phi_\text{DPO}$ in Table \ref{tab:main-results-math}), LCO‑KLD achieves superior results. This advantage stems from its direct derivation from the core reinforcement learning objective of maximizing expected advantage, providing a solid theoretical foundation that leads to stronger empirical performance.

% The superior performance of LCO-KLD can be attributed to its derivation from the original reinforcement learning objective of maximizing the expected advantage, providing a solid theoretical foundation that translates into strong empirical results.

\begin{wraptable}{r}{0.60\textwidth}
    \centering
    \renewcommand{\arraystretch}{1.0}
    \vspace{-4mm}
    \caption{Results with DPO‑RM and rule‑based RM, where the LCO advantage signal is limited to sampled action.}
    \vspace{-2mm}
    \label{tab:results-math-sparse-rm}
    \begin{adjustbox}{width=0.99\linewidth}
        \begin{tabular}{l|ccc|ccc}
        \toprule
        \multirow{2}*{\textbf{Methods}} & \multicolumn{3}{c|}{\textbf{Qwen-2.5-1.5B}} & \multicolumn{3}{c}{\textbf{Qwen-3-1.7B}} \\
         & \textbf{MATH500} & \textbf{AMC23} & \textbf{MinervaMath} & \textbf{MATH500} & \textbf{AMC23} & \textbf{MinervaMath} \\
         % & Pass@1 & Pass@1 & Pass@1 & Pass@1 & Pass@1 & Pass@1 \\
        \hline
        \hline
        \multicolumn{7}{c}{w/ \textit{DPO-RM}} \\
        \hline
        PPO & 41.30$\pm$1.32 & 24.70$\pm$2.13 & 9.23$\pm$0.71 & 55.20$\pm$1.32 & 32.20$\pm$2.31 & 15.31$\pm$1.97 \\
        \hline
        \rowcolor{mylightblue} LCO-MSE & 47.35$\pm$1.64 & 26.80$\pm$1.75 & 9.76$\pm$1.44 & 61.30$\pm$1.23 & 38.10$\pm$1.10 & 16.13$\pm$1.38 \\
        \rowcolor{mylightblue} LCO-LCH & 46.80$\pm$1.41 & 26.50$\pm$2.01 & 8.32$\pm$1.32 & 64.70$\pm$1.93 & 37.50$\pm$1.03 & 15.43$\pm$0.98 \\
        \rowcolor{mylightblue} LCO-KLD & 48.80$\pm$0.94 & 27.50$\pm$2.18 & 10.32$\pm$1.22 & 63.25$\pm$2.01 & 39.50$\pm$1.83 & 16.98$\pm$1.34 \\
        \hline
        \multicolumn{7}{c}{w/ \textit{Rule-based RM}} \\
        \hline
        GRPO & 44.70$\pm$1.39 & 22.70$\pm$2.31 & 8.65$\pm$0.83 & 56.20$\pm$1.31 & 34.50$\pm$2.27 & 14.37$\pm$1.02 \\
        \hline
        \rowcolor{mylightblue} LCO-MSE & 45.85$\pm$1.13 & 24.80$\pm$3.01 & 8.78$\pm$0.57 & 59.40$\pm$1.84 & 35.65$\pm$1.31 & 15.98$\pm$1.95 \\
        \rowcolor{mylightblue} LCO-LCH & 46.20$\pm$1.54 & 25.80$\pm$2.83 & 8.72$\pm$1.02 & 58.70$\pm$1.42 & 34.90$\pm$2.31 & 15.43$\pm$0.98 \\
        \rowcolor{mylightblue} LCO-KLD & 46.20$\pm$2.51 & 25.80$\pm$2.31 & 10.87$\pm$1.36 & 61.40$\pm$1.31 & 37.10$\pm$2.13 & 16.54$\pm$1.34 \\
        \bottomrule
        \end{tabular}
    \end{adjustbox}
    % \vspace{-2mm}
\end{wraptable}

% \vspace{-2mm}
\paragraph{LCO with Limited Advantage Feedback} 
To evaluate the robustness of the LCO framework under sparse feedback, where the advantage signal is restricted to sampled actions, we analyze its performance, presented in Table \ref{tab:results-math-sparse-rm}. Specifically, when combined with DPO-RM, all LCO variants surpass PPO across two model backbones. For example, on the Qwen-3-1.7B model, LCO-LCH achieves a Pass@1 score of 64.70\% on MATH500, significantly higher than PPO's 55.20\%. Similarly, when integrated with a rule-based RM, where a math verifier assigns a score of $+1$ for correct answers and $-1$ for incorrect ones, LCO demonstrates comparable superiority. In particular, LCO-KLD outperforms GRPO in final accuracy, achieving a MinervaMath score of 10.87\% on Qwen-2.5-1.5B, compared to GRPO's 8.65\%. These results indicate that the LCO remains effective even with sparse advantage signals, exhibiting a more robust optimization.

% \vspace{-3mm}
\section{Conclusion}
% \vspace{-3mm}
In this work, we analyze the sources of RL instability in LLMs and identified \emph{logits convexity} as essential for well-behaved gradients. We demonstrated the widely used clipped surrogate objective lacks this property, leading to gradient fluctuations and training collapse. Leveraging this insight, we proposed Logits Convex Optimization (LCO) framework, preserving logits convexity and promotes stable training.
% , mitigates sudden gradient spikes, and can be seamlessly integrated into existing RL algorithms. 
Empirical results show that LCO delivers consistently stable gradient dynamics and improved performance across various tasks and model families. Our findings provide both a theoretical explanation for RL instability and a practical framework for more reliable optimization.

\section*{Acknowledgment}

This work was supported by the National Natural Science Foundation of China (No. 62576368). 

%%% REFERNCE %%%

\begingroup
\small
\bibliographystyle{plainnat}
\bibliography{neurips_2026}

\begin{thebibliography}{39}
\providecommand{\natexlab}[1]{#1}
\providecommand{\url}[1]{\texttt{#1}}
\expandafter\ifx\csname urlstyle\endcsname\relax
  \providecommand{\doi}[1]{doi: #1}\else
  \providecommand{\doi}{doi: \begingroup \urlstyle{rm}\Url}\fi

\bibitem[Abdolmaleki et~al.(2018)Abdolmaleki, Springenberg, Tassa, Munos, Heess, and Riedmiller]{abdolmaleki2018maximum}
Abbas Abdolmaleki, Jost~Tobias Springenberg, Yuval Tassa, Remi Munos, Nicolas Heess, and Martin Riedmiller.
\newblock Maximum a posteriori policy optimisation.
\newblock \emph{arXiv preprint arXiv:1806.06920}, 2018.

\bibitem[Agarwal et~al.(2024)Agarwal, Vieillard, Zhou, Stanczyk, Garea, Geist, and Bachem]{agarwal2024policy}
Rishabh Agarwal, Nino Vieillard, Yongchao Zhou, Piotr Stanczyk, Sabela~Ramos Garea, Matthieu Geist, and Olivier Bachem.
\newblock On-policy distillation of language models: Learning from self-generated mistakes.
\newblock In \emph{The twelfth international conference on learning representations}, 2024.

\bibitem[Ahmadian et~al.(2024)Ahmadian, Cremer, Gall{\'e}, Fadaee, Kreutzer, Pietquin, {\"U}st{\"u}n, and Hooker]{ahmadian2024back}
Arash Ahmadian, Chris Cremer, Matthias Gall{\'e}, Marzieh Fadaee, Julia Kreutzer, Olivier Pietquin, Ahmet {\"U}st{\"u}n, and Sara Hooker.
\newblock Back to basics: Revisiting {REINFORCE}-style optimization for learning from human feedback in {LLM}s.
\newblock In Lun-Wei Ku, Andre Martins, and Vivek Srikumar, editors, \emph{Proceedings of the 62nd Annual Meeting of the Association for Computational Linguistics (Volume 1: Long Papers)}, pages 12248--12267, Bangkok, Thailand, August 2024. Association for Computational Linguistics.

\bibitem[Bai et~al.(2022)Bai, Jones, Ndousse, Askell, Chen, DasSarma, Drain, Fort, Ganguli, Henighan, Joseph, Kadavath, Kernion, Conerly, El-Showk, Elhage, Hatfield-Dodds, Hernandez, Hume, Johnston, Kravec, Lovitt, Nanda, Olsson, Amodei, Brown, Clark, McCandlish, Olah, Mann, and Kaplan]{bai2022training}
Yuntao Bai, Andy Jones, Kamal Ndousse, Amanda Askell, Anna Chen, Nova DasSarma, Dawn Drain, Stanislav Fort, Deep Ganguli, Tom Henighan, Nicholas Joseph, Saurav Kadavath, Jackson Kernion, Tom Conerly, Sheer El-Showk, Nelson Elhage, Zac Hatfield-Dodds, Danny Hernandez, Tristan Hume, Scott Johnston, Shauna Kravec, Liane Lovitt, Neel Nanda, Catherine Olsson, Dario Amodei, Tom Brown, Jack Clark, Sam McCandlish, Chris Olah, Ben Mann, and Jared Kaplan.
\newblock Training a helpful and harmless assistant with reinforcement learning from human feedback.
\newblock \emph{arXiv preprint arXiv:2204.05862}, 2022.

\bibitem[Chen et~al.(2025{\natexlab{a}})Chen, Li, Gong, Jiang, Fei, Yang, Shan, Yu, Wang, Zhu, et~al.]{chen2025minimax}
Aili Chen, Aonian Li, Bangwei Gong, Binyang Jiang, Bo~Fei, Bo~Yang, Boji Shan, Changqing Yu, Chao Wang, Cheng Zhu, et~al.
\newblock Minimax-m1: Scaling test-time compute efficiently with lightning attention.
\newblock \emph{arXiv preprint arXiv:2506.13585}, 2025{\natexlab{a}}.

\bibitem[Chen et~al.(2025{\natexlab{b}})Chen, Yang, Gao, Chen, Quan, Tian, and Yao]{chen2025discriminative}
Hongzhan Chen, Tao Yang, Shiping Gao, Ruijun Chen, Xiaojun Quan, Hongtao Tian, and Ting Yao.
\newblock Discriminative policy optimization for token-level reward models.
\newblock \emph{arXiv preprint arXiv:2505.23363}, 2025{\natexlab{b}}.

\bibitem[Cui et~al.(2024)Cui, Yuan, Ding, Yao, He, Zhu, Ni, Xie, Xie, Lin, Liu, and Sun]{cui2023ultrafeedback}
Ganqu Cui, Lifan Yuan, Ning Ding, Guanming Yao, Bingxiang He, Wei Zhu, Yuan Ni, Guotong Xie, Ruobing Xie, Yankai Lin, Zhiyuan Liu, and Maosong Sun.
\newblock {ULTRAFEEDBACK}: Boosting language models with scaled {AI} feedback.
\newblock In Ruslan Salakhutdinov, Zico Kolter, Katherine Heller, Adrian Weller, Nuria Oliver, Jonathan Scarlett, and Felix Berkenkamp, editors, \emph{Proceedings of the 41st International Conference on Machine Learning}, volume 235 of \emph{Proceedings of Machine Learning Research}, pages 9722--9744. PMLR, 21--27 Jul 2024.
\newblock URL \url{https://proceedings.mlr.press/v235/cui24f.html}.

\bibitem[Cui et~al.(2025)Cui, Zhang, Chen, Yuan, Wang, Zuo, Li, Fan, Chen, Chen, et~al.]{cui2025entropy}
Ganqu Cui, Yuchen Zhang, Jiacheng Chen, Lifan Yuan, Zhi Wang, Yuxin Zuo, Haozhan Li, Yuchen Fan, Huayu Chen, Weize Chen, et~al.
\newblock The entropy mechanism of reinforcement learning for reasoning language models.
\newblock \emph{arXiv preprint arXiv:2505.22617}, 2025.

\bibitem[Dubois et~al.(2024)Dubois, Li, Taori, Zhang, Gulrajani, Ba, Guestrin, Liang, and Hashimoto]{dubois2023alpacafarm}
Yann Dubois, Chen~Xuechen Li, Rohan Taori, Tianyi Zhang, Ishaan Gulrajani, Jimmy Ba, Carlos Guestrin, Percy~S Liang, and Tatsunori~B Hashimoto.
\newblock Alpacafarm: A simulation framework for methods that learn from human feedback.
\newblock \emph{Advances in Neural Information Processing Systems}, 36, 2024.

\bibitem[Gu et~al.(2023)Gu, Dong, Wei, and Huang]{gu2023minillm}
Yuxian Gu, Li~Dong, Furu Wei, and Minlie Huang.
\newblock Minillm: Knowledge distillation of large language models.
\newblock \emph{arXiv preprint arXiv:2306.08543}, 2023.

\bibitem[Guo et~al.(2025)Guo, Yang, Zhang, Song, Zhang, Xu, Zhu, Ma, Wang, Bi, et~al.]{guo2025deepseek}
Daya Guo, Dejian Yang, Haowei Zhang, Junxiao Song, Ruoyu Zhang, Runxin Xu, Qihao Zhu, Shirong Ma, Peiyi Wang, Xiao Bi, et~al.
\newblock Deepseek-r1: Incentivizing reasoning capability in llms via reinforcement learning.
\newblock \emph{arXiv preprint arXiv:2501.12948}, 2025.

\bibitem[Hu et~al.(2025)Hu, Liu, Xu, and Shen]{hu2025reinforce++}
Jian Hu, Jason~Klein Liu, Haotian Xu, and Wei Shen.
\newblock Reinforce++: An efficient rlhf algorithm with robustness to both prompt and reward models.
\newblock \emph{arXiv preprint arXiv:2501.03262}, 2025.

\bibitem[Jacot et~al.(2018)Jacot, Gabriel, and Hongler]{jacot2018neural}
Arthur Jacot, Franck Gabriel, and Cl{\'e}ment Hongler.
\newblock Neural tangent kernel: Convergence and generalization in neural networks.
\newblock \emph{Advances in neural information processing systems}, 31, 2018.

\bibitem[Lee et~al.(2019)Lee, Xiao, Schoenholz, Bahri, Novak, Sohl-Dickstein, and Pennington]{lee2019wide}
Jaehoon Lee, Lechao Xiao, Samuel Schoenholz, Yasaman Bahri, Roman Novak, Jascha Sohl-Dickstein, and Jeffrey Pennington.
\newblock Wide neural networks of any depth evolve as linear models under gradient descent.
\newblock \emph{Advances in neural information processing systems}, 32, 2019.

\bibitem[Lewkowycz et~al.(2022)Lewkowycz, Andreassen, Dohan, Dyer, Michalewski, Ramasesh, Slone, Anil, Schlag, Gutman-Solo, et~al.]{lewkowycz2022solving}
Aitor Lewkowycz, Anders Andreassen, David Dohan, Ethan Dyer, Henryk Michalewski, Vinay Ramasesh, Ambrose Slone, Cem Anil, Imanol Schlag, Theo Gutman-Solo, et~al.
\newblock Solving quantitative reasoning problems with language models.
\newblock \emph{Advances in Neural Information Processing Systems}, 35:\penalty0 3843--3857, 2022.

\bibitem[Lightman et~al.(2023)Lightman, Kosaraju, Burda, Edwards, Baker, Lee, Leike, Schulman, Sutskever, and Cobbe]{lightman2023let}
Hunter Lightman, Vineet Kosaraju, Yuri Burda, Harrison Edwards, Bowen Baker, Teddy Lee, Jan Leike, John Schulman, Ilya Sutskever, and Karl Cobbe.
\newblock Let's verify step by step.
\newblock In \emph{The Twelfth International Conference on Learning Representations}, 2023.

\bibitem[MAA(2023)]{maa2023amc}
MAA.
\newblock American mathematics competitions, 2023.
\newblock URL \url{https://www.maa.org/math-competitions}.
\newblock Online competition series.

\bibitem[Ouyang et~al.(2022)Ouyang, Wu, Jiang, Almeida, Wainwright, Mishkin, Zhang, Agarwal, Slama, Ray, et~al.]{ouyang2022training}
Long Ouyang, Jeffrey Wu, Xu~Jiang, Diogo Almeida, Carroll Wainwright, Pamela Mishkin, Chong Zhang, Sandhini Agarwal, Katarina Slama, Alex Ray, et~al.
\newblock Training language models to follow instructions with human feedback.
\newblock \emph{Advances in Neural Information Processing Systems}, 35:\penalty0 27730--27744, 2022.

\bibitem[Peng et~al.(2019)Peng, Kumar, Zhang, and Levine]{peng2019advantage}
Xue~Bin Peng, Aviral Kumar, Grace Zhang, and Sergey Levine.
\newblock Advantage-weighted regression: Simple and scalable off-policy reinforcement learning.
\newblock \emph{arXiv preprint arXiv:1910.00177}, 2019.

\bibitem[Rafailov et~al.(2024{\natexlab{a}})Rafailov, Hejna, Park, and Finn]{rafailov2024from}
Rafael Rafailov, Joey Hejna, Ryan Park, and Chelsea Finn.
\newblock From \$r\$ to \$q{\textasciicircum}*\$: Your language model is secretly a q-function.
\newblock In \emph{First Conference on Language Modeling}, 2024{\natexlab{a}}.
\newblock URL \url{https://openreview.net/forum?id=kEVcNxtqXk}.

\bibitem[Rafailov et~al.(2024{\natexlab{b}})Rafailov, Sharma, Mitchell, Manning, Ermon, and Finn]{rafailov2024direct}
Rafael Rafailov, Archit Sharma, Eric Mitchell, Christopher~D Manning, Stefano Ermon, and Chelsea Finn.
\newblock Direct preference optimization: Your language model is secretly a reward model.
\newblock \emph{Advances in Neural Information Processing Systems}, 36, 2024{\natexlab{b}}.

\bibitem[Saleh and Saleh(2022)]{saleh2022statistical}
Resve~A Saleh and AK~Saleh.
\newblock Statistical properties of the log-cosh loss function used in machine learning.
\newblock \emph{arXiv preprint arXiv:2208.04564}, 2022.

\bibitem[Schulman et~al.(2015{\natexlab{a}})Schulman, Levine, Abbeel, Jordan, and Moritz]{schulman2015trust}
John Schulman, Sergey Levine, Pieter Abbeel, Michael Jordan, and Philipp Moritz.
\newblock Trust region policy optimization.
\newblock In \emph{International conference on machine learning}, pages 1889--1897. PMLR, 2015{\natexlab{a}}.

\bibitem[Schulman et~al.(2015{\natexlab{b}})Schulman, Moritz, Levine, Jordan, and Abbeel]{schulman2015high}
John Schulman, Philipp Moritz, Sergey Levine, Michael Jordan, and Pieter Abbeel.
\newblock High-dimensional continuous control using generalized advantage estimation.
\newblock \emph{arXiv preprint arXiv:1506.02438}, 2015{\natexlab{b}}.

\bibitem[Schulman et~al.(2017)Schulman, Wolski, Dhariwal, Radford, and Klimov]{schulman2017proximal}
John Schulman, Filip Wolski, Prafulla Dhariwal, Alec Radford, and Oleg Klimov.
\newblock Proximal policy optimization algorithms.
\newblock \emph{arXiv preprint arXiv:1707.06347}, 2017.

\bibitem[Shao et~al.(2024)Shao, Wang, Zhu, Xu, Song, Bi, Zhang, Zhang, Li, Wu, et~al.]{shao2024deepseekmath}
Zhihong Shao, Peiyi Wang, Qihao Zhu, Runxin Xu, Junxiao Song, Xiao Bi, Haowei Zhang, Mingchuan Zhang, YK~Li, Y~Wu, et~al.
\newblock Deepseekmath: Pushing the limits of mathematical reasoning in open language models.
\newblock \emph{arXiv preprint arXiv:2402.03300}, 2024.

\bibitem[Team et~al.(2025)Team, Hu, Chen, Zhao, Liu, Jin, Zhu, Dai, Luan, Guo, et~al.]{team2025ring}
Ling Team, Bin Hu, Cai Chen, Deng Zhao, Ding Liu, Dingnan Jin, Feng Zhu, Hao Dai, Hongzhi Luan, Jia Guo, et~al.
\newblock Ring-lite: Scalable reasoning via c3po-stabilized reinforcement learning for llms.
\newblock \emph{arXiv preprint arXiv:2506.14731}, 2025.

\bibitem[Tomar et~al.(2020)Tomar, Shani, Efroni, and Ghavamzadeh]{tomar2020mirror}
Manan Tomar, Lior Shani, Yonathan Efroni, and Mohammad Ghavamzadeh.
\newblock Mirror descent policy optimization.
\newblock \emph{arXiv preprint arXiv:2005.09814}, 2020.

\bibitem[Williams(1992)]{williams1992simple}
Ronald~J Williams.
\newblock Simple statistical gradient-following algorithms for connectionist reinforcement learning.
\newblock \emph{Machine learning}, 8:\penalty0 229--256, 1992.

\bibitem[Wu et~al.(2023)Wu, Hu, Shi, Dziri, Suhr, Ammanabrolu, Smith, Ostendorf, and Hajishirzi]{wu2023fine}
Zeqiu Wu, Yushi Hu, Weijia Shi, Nouha Dziri, Alane Suhr, Prithviraj Ammanabrolu, Noah~A Smith, Mari Ostendorf, and Hannaneh Hajishirzi.
\newblock Fine-grained human feedback gives better rewards for language model training.
\newblock \emph{Advances in Neural Information Processing Systems}, 36:\penalty0 59008--59033, 2023.

\bibitem[Yang et~al.(2025{\natexlab{a}})Yang, Li, Yang, Zhang, Hui, Zheng, Yu, Gao, Huang, Lv, et~al.]{yang2025qwen3}
An~Yang, Anfeng Li, Baosong Yang, Beichen Zhang, Binyuan Hui, Bo~Zheng, Bowen Yu, Chang Gao, Chengen Huang, Chenxu Lv, et~al.
\newblock Qwen3 technical report.
\newblock \emph{arXiv preprint arXiv:2505.09388}, 2025{\natexlab{a}}.

\bibitem[Yang et~al.(2025{\natexlab{b}})Yang, Dou, Guo, Lu, Ju, Deng, and Xin]{yang2025dcpo}
Shihui Yang, Chengfeng Dou, Peidong Guo, Kai Lu, Qiang Ju, Fei Deng, and Rihui Xin.
\newblock Dcpo: Dynamic clipping policy optimization.
\newblock \emph{arXiv preprint arXiv:2509.02333}, 2025{\natexlab{b}}.

\bibitem[Yu et~al.(2025)Yu, Zhang, Zhu, Yuan, Zuo, Yue, Dai, Fan, Liu, Liu, et~al.]{yu2025dapo}
Qiying Yu, Zheng Zhang, Ruofei Zhu, Yufeng Yuan, Xiaochen Zuo, Yu~Yue, Weinan Dai, Tiantian Fan, Gaohong Liu, Lingjun Liu, et~al.
\newblock Dapo: An open-source llm reinforcement learning system at scale.
\newblock \emph{arXiv preprint arXiv:2503.14476}, 2025.

\bibitem[Yuan et~al.(2025)Yuan, Yue, Zhu, Fan, and Yan]{yuan2025what}
Yufeng Yuan, Yu~Yue, Ruofei Zhu, Tiantian Fan, and Lin Yan.
\newblock What's behind ppo's collapse in long-cot? value optimization holds the secret.
\newblock \emph{arXiv preprint arXiv:2503.01491}, 2025.

\bibitem[Zhang et~al.(2025)Zhang, Lu, Hu, Fu, Wen, Zhang, Liu, Jiang, Chen, Tang, et~al.]{zhang2025r1}
Yi-Fan Zhang, Xingyu Lu, Xiao Hu, Chaoyou Fu, Bin Wen, Tianke Zhang, Changyi Liu, Kaiyu Jiang, Kaibing Chen, Kaiyu Tang, et~al.
\newblock R1-reward: Training multimodal reward model through stable reinforcement learning.
\newblock \emph{arXiv preprint arXiv:2505.02835}, 2025.

\bibitem[Zheng et~al.(2025)Zheng, Liu, Li, Chen, Yu, Gao, Dang, Liu, Men, Yang, et~al.]{zheng2025group}
Chujie Zheng, Shixuan Liu, Mingze Li, Xiong-Hui Chen, Bowen Yu, Chang Gao, Kai Dang, Yuqiong Liu, Rui Men, An~Yang, et~al.
\newblock Group sequence policy optimization.
\newblock \emph{arXiv preprint arXiv:2507.18071}, 2025.

\bibitem[Zhong et~al.(2024)Zhong, Shan, Feng, Xiong, Cheng, Zhao, He, Bian, and Wang]{zhong2024dpo}
Han Zhong, Zikang Shan, Guhao Feng, Wei Xiong, Xinle Cheng, Li~Zhao, Di~He, Jiang Bian, and Liwei Wang.
\newblock Dpo meets ppo: Reinforced token optimization for rlhf.
\newblock \emph{arXiv preprint arXiv:2404.18922}, 2024.

\bibitem[Zhu et~al.(2025)Zhu, Liu, Zhang, Wu, and Zhang]{zhu2025carft}
Wenqiao Zhu, Ji~Liu, Rongjuncheng Zhang, Haipang Wu, and Yulun Zhang.
\newblock Carft: Boosting llm reasoning via contrastive learning with annotated chain-of-thought-based reinforced fine-tuning.
\newblock \emph{arXiv preprint arXiv:2508.15868}, 2025.

\bibitem[Ziebart(2010)]{ziebart2010modeling}
Brian~D Ziebart.
\newblock \emph{Modeling purposeful adaptive behavior with the principle of maximum causal entropy}.
\newblock Carnegie Mellon University, 2010.

\end{thebibliography}
\endgroup

%%%%%%%%%%%%%%%%%%%%%%%%%%%%%%%%%%%%%%%%%%%%%%%%%%%%%%%%%%%%
\clearpage
\appendix

% \section{Limitations}

% Despite the advantages of LCO, this work has certain limitations that should be acknowledged. 

% First, our theoretical framework is primarily built upon the assumption of over-parameterization. This condition is generally satisfied by large-scale models with several billion parameters or more, where the model capacity is sufficient to approximate the desired loss geometry. Consequently, the theoretical guarantees and the resulting LCO framework may not directly generalize to smaller models with limited parameter counts, where the over-parameterization assumption may not hold.

% Second, while we identify loss geometry as a crucial factor for training stability, it is not the sole contributor to reinforcement learning instability. Other factors, such as the off-policy effect, which arises from the discrepancy between the behavior policy and the target policy, also play a significant role in training dynamics. This study focuses predominantly on the geometric perspective of the loss landscape. Thus, LCO is currently designed to be employed in an on-policy manner to ensure optimal stability and convergence. Future research could explore the extension of logits convexity principles to off-policy settings and investigate its efficacy across a broader range of model architectures.

\section{Limitations}

Despite the advantages of LCO, this work has several limitations that should be acknowledged.

First, our theoretical framework is primarily built upon the assumption of over-parameterization. This condition is generally satisfied by large-scale models with several billion parameters or more, where the model capacity is sufficient to support the theoretical analysis. Consequently, the theoretical guarantees and the resulting LCO framework may not directly generalize to smaller models with limited parameter counts, where the over-parameterization assumption may not hold.

Second, RL instability is multifaceted and cannot be attributed to a single factor. In addition to the optimization geometry induced by the training objective, factors such as importance sampling variance, off-policy mismatch between the behavior and target policies, and noisy advantage estimation can affect training dynamics. Our work identifies logits convexity as an important, previously underexplored factor that influences gradient stability, rather than claiming it to be the sole cause of RL instability. LCO does not eliminate instability arising from other sources. 

\section{Related Work}

Recent research in reinforcement learning has increasingly focused on improving the stability of policy training. These efforts can be broadly categorized into three groups.

The \textbf{first category} aims to reduce the variance or bias in advantage estimation. A seminal work in this line is the GAE \citep{schulman2015high}, which combines Monte Carlo returns and a value model to balance bias and variance. 
Extending GAE, VC-PPO \citep{yuan2025what} identifies a failure mode where the value model exhibits bias during training, resulting in large errors in advantage estimation. To address this, they propose a pretraining procedure for the value model, and decouple the $\lambda$ in GAE for the policy and value model computations.
\citet{zhang2025r1} identify outliers caused by the imbalance in the advantage distribution. They propose StableReinforce, which applies an advantage filter to retain only those advantages that fall within three standard deviations for stable training.
By simplifying the advantage estimation process, RLOO \citep{ahmadian2024back} employs a leave-one-out baseline across multiple completions to produce advantage estimate for prompt. Similarly, \citet{shao2024deepseekmath} introduce GRPO, which standardizes sequence-level rewards by subtracting the mean and dividing by the standard deviation, thereby reducing bias and variance. Extending GRPO, \citet{yu2025dapo} propose DAPO, which re-weights token-level losses to prevent longer responses from being underrepresented in gradient updates.

The \textbf{second category} stabilizes training by constraining policy updates through a Kullback-Leibler (KL) divergence penalty relative to a reference model. 
For example, TRPO \citep{schulman2015trust} aims to find a policy that increases the probability of advantageous actions while limiting the divergence from the previous policy using a KL constraint, ensuring stable training. Building upon PPO, \citet{ouyang2022training,hu2025reinforce++} add a token-level KL penalty to the reward function, which constrains the policy at each generation step to remain close to the reference SFT model.
GRPO \citep{shao2024deepseekmath} modifies this approach by applying the KL constraint directly to the policy loss rather than the reward, which allows for more targeted optimization. KL-Cov \citep{cui2025entropy} advances this idea by analyzing policy entropy, showing that entropy change is driven by the covariance between action probabilities and advantages, and applying KL penalties selectively to high-covariance tokens to prevent entropy collapse and improve stability.

The \textbf{third category} employs clipping mechanisms to stabilize policy updates.
PPO and GRPO constrain the importance sampling ratio between current and previous policies within fixed upper and lower bounds to prevent excessively large policy updates. However, such bounds can limit training efficiency and unduly constrain specific updates. To address this, DAPO \citep{yu2025dapo} proposes a decoupled clip-higher method that relaxes the upper clipping bound to improve training efficiency while maintaining stability. 
Building upon the same idea, DCPO \citep{yang2025dcpo} addresses the limitation in DAPO, where the same clip range is set for different positions. It further introduces a dynamic clipping method that adaptively adjusts the clipping bounds based on the token probabilities, thereby mitigating the drawbacks of fixed clipping bounds.
\citet{chen2025minimax} identify a key limitation in PPO/GRPO: clipping can prematurely drop high-advantage tokens from contributing to off-policy gradients. They introduce CISPO, which clips importance sampling weights without clipping token updates to stabilize training. Extending this covariance analysis, \citet{cui2025entropy} propose Clip-Cov, which applies clipping selectively to updates on high-covariance tokens to further enhance training stability.

Beyond these approaches to stabilizing policy optimization, mirror descent and policy improvement methods provide a general framework for constructing improved policies and fitting them to parameterized policy. MDPO \citep{tomar2020mirror} formulates policy optimization through iterative mirror-descent updates, while AWR \citep{peng2019advantage} and MPO \citep{abdolmaleki2018maximum} fit improved policies using objectives based on weighted maximum likelihood or KL divergence. In contrast to these methods, LCO reformulates the fitting step directly in logit space by optimizing a unified family of convex logit-matching objectives. Unlike conventional supervised fitting of the improved policy, LCO explicitly enforces desirable convexity properties in the optimization objective, which enables our gradient analysis and the associated optimization guarantees.

% Unlike previous work, our study is inspired by the stable training of SFT and provides a theoretical analysis of RL instability from a gradient perspective. We identify \textit{logits convexity}, which induces smoother gradient updates during optimization and ensures more stable RL training, and propose a simple yet effective policy optimization framework.

\begin{figure*}[b]
    \centering
    \vspace{-1mm}
    \includegraphics[width=0.99\linewidth]{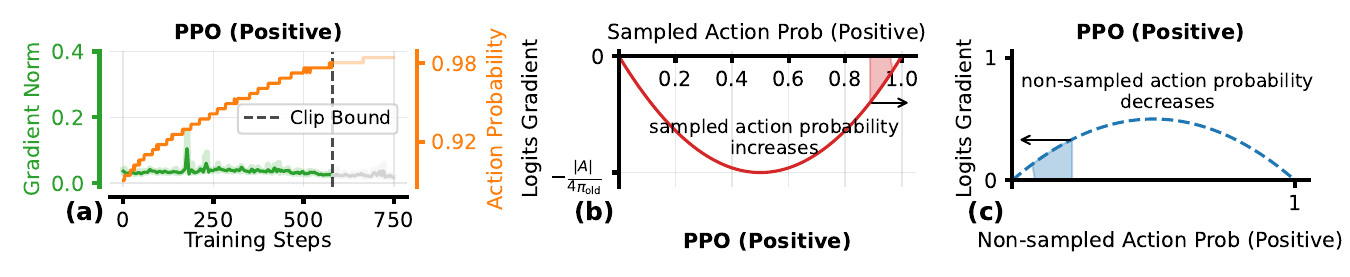}
    \vspace{-4mm}
    \caption{Training dynamics of PPO for actions with positive advantages. \textbf{(a)} Gradient norm decreases as training progresses while sampled action probabilities increase. The magnitudes of both sampled \textbf{(b)} and non-sampled action logit gradient \textbf{(c)} decrease.}
    \label{fig:introduction-ratio-positive}
    \vspace{-4mm}
\end{figure*}

\section{Empirical Validation of the Local Linearization Assumption}
\label{app:local_linearization}

Assumption \ref{ass:second-order-assumption} relies on the local linearization of the neural network
around its initialization. To empirically examine the validity of this
assumption in our setting, we measure the relative parameter displacement
during PPO training of Qwen3-4B, initialized from a converged SFT checkpoint.
Let $\theta_0$ denote the parameters of the initial SFT checkpoint and
$\theta_t$ denote the parameters after $t$ training steps. We define the
relative parameter displacement as $\delta_t=\frac{\|\theta_t-\theta_0\|_2}{\|\theta_0\|_2}$. This quantity measures the magnitude of the parameter change relative to the norm of the initialization. We record the relative parameter displacement at different stages of PPO
training. As shown in Table~\ref{tab:parameter_displacement}, the
displacement increases gradually throughout training but remains below
$5\times 10^{-4}$ at all reported checkpoints.

\begin{table}[h]
    \centering
    \vspace{-3mm}
    \caption{Relative parameter displacement during PPO training of Qwen3-4B.}
    \label{tab:parameter_displacement}
    \begin{tabular}{cc}
        \toprule
        Training Steps & Relative Parameter Displacement $\delta_t$ \\
        \midrule
        1k & $9.78\times 10^{-5}$ \\
        2k & $1.27\times 10^{-4}$ \\
        3k & $2.47\times 10^{-4}$ \\
        4k & $4.23\times 10^{-4}$ \\
        \bottomrule
    \end{tabular}
\end{table}

These results show that PPO post-training induces only a small relative
change in the model parameters compared with the initialization.
This observation is consistent with the local nature of our theoretical
analysis and supports the plausibility of applying a first-order
approximation around the SFT checkpoint in our setting. In particular,
starting from a converged SFT model and performing post-training with
relatively small learning rates ($5\times 10^{-7}$ to $5\times 10^{-6}$)
helps keep the optimization trajectory close to its initialization. We emphasize that a small relative parameter displacement alone does
not guarantee that higher-order terms in the network's Taylor expansion
are negligible. Nevertheless, the observed parameter stability provides
empirical evidence supporting the use of the local linearization
assumption in our PPO post-training regime.

\section{Additional Experiment}

\paragraph{Gradient Dynamics of PPO} Figure \ref{fig:introduction-ratio-positive} shows the gradient dynamics of PPO for actions with positive advantage. We consider $A(s_t,a_t)>0$, the gradient norm of the surrogate objective decreases as training progress, as shown in Figure \ref{fig:introduction-ratio-positive}(a), while the magnitude of logit gradients decrease before being clipped (Figure \ref{fig:introduction-ratio-positive}(b) and (c)).

\paragraph{Gradient Dynamics of LCO-KLD} Figure \ref{fig:introduction-convex-positive} shows the gradient dynamics of LCO-KLD for actions with positive advantage. We consider $A(s_t,a_t)$ for the sampled action, the gradient norm of the LCO-KLD objective decreases as training progress, as shown in Figure \ref{fig:introduction-convex-positive}(a), while the magnitude of logit gradients decrease correspondingly (Figure \ref{fig:introduction-convex-positive}(b) and (c)).

\begin{figure*}[t]
    \centering
    \vspace{-0mm}
    \includegraphics[width=0.99\linewidth]{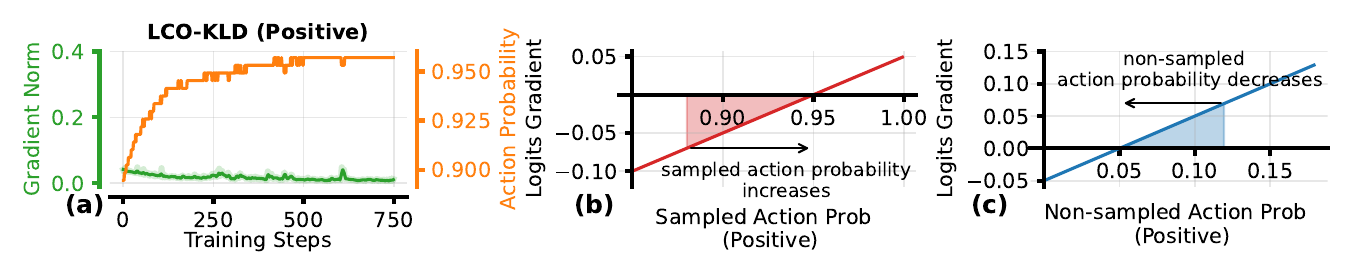}
    \vspace{-4mm}
    \caption{Training dynamics of LCO-KLD for actions with positive advantages. We show the dynamics of gradient norm \textbf{(a)}, logit gradient for sampled \textbf{(b)} and non-sampled \textbf{(c)} actions. LCO‑KLD exhibits stable gradient updates as the model converges.}
    \label{fig:introduction-convex-positive}
    \vspace{-3mm}
\end{figure*}

\begin{table}[h]
    \centering
    \renewcommand{\arraystretch}{1.0}
    \vspace{-0mm}
    \caption{OOD results. Best performances are shown in \textbf{bold}, while suboptimal ones are \underline{underlined}.}
    % \vspace{-1mm}
    \label{tab:ood-results}
    \begin{adjustbox}{width=0.55\linewidth}
        \begin{tabular}{l|cc|cc}
            \toprule
             \multirow{2}*{\textbf{Methods}} & \multicolumn{2}{c|}{\textbf{Qwen-3-4B}} & \multicolumn{2}{c}{\textbf{Qwen-2.5-3B}} \\
             & \textbf{MMLU} & \textbf{MMLU-Redux} & \textbf{MMLU} & \textbf{MMLU-Redux} \\
             \hline 
             \hline
             % REINFORCE & 71.11$\pm$0.93 & 75.07$\pm$1.18 & 57.69$\pm$2.23 & 55.18$\pm$2.85 \\
             PPO & 67.87$\pm$1.23 & 73.93$\pm$1.53 & 58.49$\pm$1.89 & 55.59$\pm$2.01 \\
             GRPO & 64.69$\pm$1.42 & 71.35$\pm$1.51 & 58.61$\pm$1.42 & 56.97$\pm$1.94 \\
             DAPO & 65.39$\pm$1.21 & 70.37$\pm$1.82 & \underline{59.94$\pm$1.11} & 52.49$\pm$1.34 \\
             GSPO & 68.71$\pm$2.01 & 72.72$\pm$1.71 & 54.95$\pm$1.38 & 56.44$\pm$1.74 \\
             \hline
             \rowcolor{mylightblue} LCO-MSE & 68.71$\pm$1.42 & \underline{74.21$\pm$1.34} & 58.10$\pm$1.67 & 56.73$\pm$2.01 \\
             \rowcolor{mylightblue} LCO-LCH & \underline{69.12$\pm$1.31} & 74.12$\pm$1.23 & 59.23$\pm$2.13 & \textbf{58.71$\pm$1.83} \\
             \rowcolor{mylightblue} LCO-KLD & \textbf{72.11$\pm$1.12} & \textbf{76.71$\pm$1.35} & \textbf{60.14$\pm$1.35} & \underline{57.58$\pm$1.67} \\
             \bottomrule
        \end{tabular}
    \end{adjustbox}
    \vspace{-3mm}
\end{table}

\paragraph{Out-of-Distribution Performance} In order to assess the out‑of‑distribution (OOD) generalization beyond the mathematical reasoning results in Table \ref{tab:ood-results}, we evaluate the models on multi‑task language understanding benchmarks, including MMLU and MMLU‑Redux. As shown in Table \ref{tab:ood-results}, the LCO objectives consistently outperform baseline methods in accuracy. With the Qwen‑3‑4B backbone, LCO‑KLD obtains the highest accuracy on both MMLU (72.11\%) and MMLU‑Redux (76.71\%), surpassing PPO (67.87\% and 73.93\%, respectively). Using the Qwen‑2.5‑3B backbone, LCO‑KLD delivers the best accuracy on MMLU (60.14\%), while LCO‑LCH achieves the top accuracy on MMLU‑Redux (58.71\%). These results highlight the strong OOD generalization and robustness of the LCO objectives across different model families.

\section{Logit Gradient}

\subsection{Logit Gradient of SFT}
\label{sec:appendix-sft-logit-gradient-derivation}
We provide the derivation for Equation \eqref{eq:sft-logit-gradient}. At a time step $t$, let $a_t$ denote the target (ground-truth) token and $a'_t$ be an arbitrary token in the vocabulary $\mathcal{V}$, the derivative of $\mathcal{L}_\text{SFT}$ defined in Equation \eqref{eq:sft-objective} with respect to the logit $z_\theta(s_t,a'_t)$ is:
\begin{equation}
    \begin{aligned}
        \frac{\partial \mathcal{L}_\text{SFT}}{\partial z_\theta(s_t,a'_t)}&=-\frac{\partial \log\pi_\theta(a_t|s_t)}{\partial\pi_\theta(a_t|s_t)}\frac{\partial\pi_\theta(a_t|s_t)}{z_\theta(s_t,a'_t)} \\
        &=-\frac{1}{\pi_\theta(a_t|s_t)}\pi_\theta(a_t|s_t)(\mathbb{I}_{a_t=a'_t}-\pi_\theta(a'_t|s_t)) \\
        &=\pi_\theta(a'_t|s_t)-\mathbb{I}_{a_t=a'_t},
    \end{aligned}
\end{equation}
where $\mathbb{I}_{a_t=a'_t}$ is the indicator function, which equals $1$ when $a_t=a'_t$, and $0$ otherwise.

\subsection{Logit Gradient of PPO}
\label{sec:appendix-ppo-logit-gradient-derivation}
We provide the derivation for Equation \eqref{eq:ppo-logit-gradient}. At a time step $t$, let $a_t$ be the sampled action and $a'_t$ be an arbitrary action in the $\mathcal{V}$, the derivative of $\mathcal{L}_\text{PPO}$ defined in Equation \eqref{eq:ppo-objective} with respect to the logit $z_\theta(s_t,a'_t)$ is (when the gradient is not clipped):
\begin{equation}
    \begin{aligned}
        \frac{\partial\mathcal{L}_\text{PPO}}{\partial z_\theta(s_t,a'_t)}&=-\frac{A(s_t,a_t)}{\pi_\text{old}(a_t|s_t)}\frac{\partial \pi_\theta(a_t|s_t)}{\partial z_\theta(s_t,a'_t)} \\
        &=\frac{A(s_t,a_t)}{\pi_\text{old}(a_t|s_t)}\pi_\theta(a_t|s_t)(\pi_\theta(a'_t|s_t)-\mathbb{I}_{a_t=a'_t}).
    \end{aligned}
\end{equation}

\section{Derivation of Optimal Policy}
\label{sec:proof-of-trpo-optimal-policy}
In this section, we provide the derivation for Equation \eqref{eq:trpo-optimal-policy}. Under the reinforcement learning objective, we aim to maximize the expected advantage with a KL constraint at each time step $t$:
\begin{equation}
    \max_{\pi_\theta} \mathbb{E}_{a_t\sim\pi_\theta(\cdot|s_t)}\left[A(s_t,a_t)\right]-\beta\mathbb{D}_\text{KL}(\pi_\theta(\cdot|s_t)\|\pi_\text{old}(\cdot|s_t)) ~~~~ s.t. ~ \sum_{a_t\in\mathcal{V}}\pi_\theta(a_t|s_t)=1.
\end{equation}
To find the optimal policy $\pi^*$, we aim to maximize above objective under the constraint $\sum_{a_t\in\mathcal{V}}\pi_\theta(a_t|s_t)=1$. To solve the constrained optimization problem, we can use the method of Lagrange multiplier. Let $\lambda$ be the Lagrange multiplier, the objective can be written as:
\begin{equation}
    \begin{aligned}
        \mathcal{J}&=\mathbb{E}_{a_t\sim\pi_\theta(\cdot|s_t)}\left[A(s_t,a_t)\right]-\beta\mathbb{D}_\text{KL}(\pi_\theta(\cdot|s_t)\|\pi_\text{old}(\cdot|s_t))+\lambda\left[\sum_{a_t\in\mathcal{V}}\pi_\theta(a_t|s_t)-1\right] \\
        &=\sum_{a_t\in\mathcal{V}}\pi_\theta(a_t|s_t)A(s_t,a_t)-\beta\left[\sum_{a_t\in\mathcal{V}}\pi_\theta(a_t|s_t)\log\frac{\pi_\theta(a_t|s_t)}{\pi_\text{old}(a_t|s_t)}\right]+\lambda\left[\sum_{a_t\in\mathcal{V}}\pi_\theta(a_t|s_t)-1\right].
    \end{aligned}
\end{equation}
We set the following partial derivative to zero and have:
\begin{equation}
    \label{eq:optimal-partial-derivative}
    \frac{\partial\mathcal{J}}{\partial \pi_\theta(a_t|s_t)}=A(s_t,a_t)-\beta\left[\log\frac{\pi_\theta(a_t|s_t)}{\pi_\text{old}(a_t|s_t)}+1\right]+\lambda=0.
\end{equation}
% TODO: 最优policy的定义
Using Equation \eqref{eq:optimal-partial-derivative}, we have the expression for optimal policy $\pi^*$ as follows:
\begin{equation}
    \pi^*(a_t|s_t)=\pi_\text{old}(a_t|s_t)\exp\left[\frac{A(s_t,a_t)}{\beta}\right]\exp\left[\frac{\lambda-\beta}{\beta}\right].
\end{equation}
Since the term $\exp\left[\frac{\lambda-\beta}{\beta}\right]$ is constant, the optimal policy $\pi^*$ is prop to:
\begin{equation}
    \pi^*(a_t|s_t)\propto\pi_\text{old}(a_t|s_t)\exp\left[\frac{A(s_t,a_t)}{\beta}\right].
\end{equation}
By ensuring the property of probability, the optimal policy can be written as:
\begin{equation}
    \label{eq:appendix-trpo-optimal-policy}
    \pi^*(a_t|s_t)=\frac{\pi_\text{old}(a_t|s_t)\exp\left[\frac{A(s_t,a_t)}{\beta}\right]}{\sum_{a'_t\in\mathcal{V}}\pi_\text{old}(a'_t|s_t)\exp\left[\frac{A(s_t,a'_t)}{\beta}\right]}.
\end{equation}
By substituting the softmax definition of the behavioral policy $\pi_\text{old}(a_t|s_t)$ into Equation \eqref{eq:appendix-trpo-optimal-policy}, the normalization terms $\sum_{a'_t\in\mathcal{V}}\exp z_\text{old}(s_t,a'_t)$ in the numerator and denominator cancel out. This allows us to reformulate the optimal policy directly in terms of the behavioral logits $z_\text{old}(s_t,a_t)$:
% By canceling out the normalized term $\sum_{a'_t\in\mathcal{V}}\exp z_\text{old}(s_t,a'_t)$ in $\pi_\text{old}(a_t|s_t)$, we reformulate the optimal policy:
\begin{equation}
    \begin{aligned}
        \pi^*(a_t|s_t)=\frac{\exp\left[z_\text{old}(s_t,a_t)+\frac{A(s_t,a_t)}{\beta}\right]}{\sum_{a'_t\in\mathcal{V}}\exp\left[z_\text{old}(s_t,a'_t)+\frac{A(s_t,a'_t)}{\beta}\right]}.
    \end{aligned}
\end{equation}
By equating the softmax arguments, we can identify a solution for the optimal logits $z^*$ as:
\begin{equation}
    z^*(s_t,a_t)=z_\text{old}(s_t,a_t)+\frac{A(s_t,a_t)}{\beta}.
\end{equation}
These optimal logits are obtained via a direct advantage-based adjustment without an additional state-dependent constant shift. For clarity and consistency, we define this specific solution as the unique optimal logits throughout this paper.

\section{Derivation of Gradient Directionality}
\label{sec:proof-of-gradient-directionality}

In this section, we provide the derivation for Proposition \ref{prop:gradient-directionality}. 
Consider logits $\boldsymbol{z}_\theta$ parameterized by $\theta$. Let $\Delta\theta=\theta^*-\theta$ be the parameter displacement, the first-order Taylor expansion around $\theta$ is:
\begin{equation}
    \boldsymbol{z}_{\theta^*}=\boldsymbol{z}_{\theta}+\nabla_\theta \boldsymbol{z}_{\theta}^\top\Delta\theta+o(\|\Delta\theta\|).
\end{equation}
Under Assumption \ref{ass:second-order-assumption}, the higher-order term $o(\|\Delta\theta\|)$ can be truncated without significant loss of accuracy:
\begin{equation}
    \label{eq:logits-first-order-approximation}
    \boldsymbol{z}_{\theta^*}\approx\boldsymbol{z}_{\theta}+\nabla_\theta \boldsymbol{z}_{\theta}^\top\Delta\theta.
\end{equation}
Let $\mathcal{L}$ be a differentiable loss function that take $\boldsymbol{z}_\theta$ as input. By applying the chain rule, we evaluate the inner product between the parameter gradient and the parameter displacement:
\begin{equation}
    \langle\nabla_\theta\mathcal{L},\theta-\theta^*\rangle=\langle\nabla_\theta\boldsymbol{z}_\theta\nabla_{\boldsymbol{z}_\theta}\mathcal{L},\theta-\theta^*\rangle=\langle\nabla_{\boldsymbol{z}_\theta}\mathcal{L},\nabla_\theta\boldsymbol{z}_\theta^\top(\theta-\theta^*)\rangle.
\end{equation}
By rearranging the terms in Equation \eqref{eq:logits-first-order-approximation}, and substituting into the inner product expression yields:
\begin{equation}
    \langle\nabla_{\boldsymbol{z}_\theta}\mathcal{L},\nabla_\theta\boldsymbol{z}_\theta^\top(\theta-\theta^*)\rangle\approx\langle\nabla_{\boldsymbol{z}_\theta}\mathcal{L},\boldsymbol{z}_\theta-\boldsymbol{z}_{\theta^*}\rangle.
\end{equation}
If $\mathcal{L}$ is convex with respect to logits, by the first-order characterization of convex function, we have:
\begin{equation}
    \langle\nabla_{\boldsymbol{z}_\theta}\mathcal{L},\nabla_\theta\boldsymbol{z}_\theta^\top(\theta-\theta^*)\rangle\approx\langle\nabla_{\boldsymbol{z}_\theta}\mathcal{L},\boldsymbol{z}_\theta-\boldsymbol{z}_{\theta^*}\rangle\geq0.
\end{equation}

\section{Logits Convexity of Different Objectives}
\label{sec:proof-of-convexity}

\subsection{Logits Convexity of SFT and PPO Losses}
\label{sec:proof-of-sft-convexity}
\paragraph{Logits Convexity of SFT} At time step $t$, the Hessian $\boldsymbol{H}_\text{SFT}$ of $\mathcal{L}_\text{SFT}$ (Equation \eqref{eq:sft-objective}) with respect to the logits $\boldsymbol{z}_\theta$ is given by:
\begin{equation}
    \boldsymbol{H}_\text{SFT}=\text{diag}(\boldsymbol{\pi}_\theta)-\boldsymbol{\pi}_\theta\boldsymbol{\pi}_\theta^\top,
\end{equation}
where $\boldsymbol{\pi}_\theta\in\mathbb{R}^{|\mathcal{V}|}$ denotes the action probability distribution over the vocabulary $\mathcal{V}$ at time step $t$, $\text{diag}(\boldsymbol{\pi}_\theta)$ is a diagonal matrix with the elements of $\boldsymbol{\pi}_\theta$ along its main diagonal. For any vector $\boldsymbol{v}\in\mathbb{R}^{|\mathcal{V}|}$, consider the quadratic form:
\begin{equation}
    \label{eq:sft-psd}
    \boldsymbol{v}^\top\boldsymbol{H}_\text{SFT}\boldsymbol{v}=\sum^{|\mathcal{V}|}_i\pi_{\theta,i}v^2_i-\left(\sum^{|\mathcal{V}|}_i\pi_{\theta,i}v_i\right)^2\geq0,
\end{equation}
where the final inequality follows from the Cauchy-Schwarz inequality. This demonstrates that the Hessian matrix $\boldsymbol{H}_\text{SFT}$ is positive semi-definite (PSD), which in turn implies that $\mathcal{L}_\text{SFT}$ is logits convex at each time step.

% \subsection{Logits Convexity of PPO Loss}
% \label{sec:proof-of-ppo-convexity}

\paragraph{Logits Convexity of PPO} The gradient of $\mathcal{L}_\text{PPO}$, as defined in Equation \eqref{eq:ppo-objective}, is non-zero only when the condition $(A(s_t,a_t) > 0 \text{ and } r_t(\theta)<1+\epsilon) \text{ or } (A(s_t,a_t)<0 \text{ and } r_t(\theta) > 1 - \epsilon)$ is satisfied. If this condition is not met, the objective enters the clipped region, and the gradient becomes zero. Therefore, we focus our analysis on the active region where the gradient is non-zero. Under this setting, the gradient of $\mathcal{L}_\text{PPO}$ with respect to a logit $z_\theta(s_t,a_t')$ at time step $t$ is:
\begin{equation}
    \frac{\partial \mathcal{L}_\text{PPO}}{\partial z_\theta(s_t,a_t')}=-\frac{A(s_t,a_t)}{\pi_\text{old}(a_t|s_t)}\pi_\theta(a_t|s_t)(\mathbb{I}_{a_t=a'_t}-\pi_\theta(a'_t|s_t)).
\end{equation}
And the second derivative is as follow:
\begin{equation}
    \label{eq:ppo-second-derivative}
    \begin{aligned}
        &\frac{\partial^2\mathcal{L}_\text{PPO}}{\partial z_\theta(s_t,a'_t)\partial z_\theta(s_t,a''_t)}= \\
        & -\frac{A(s_t,a_t)}{\pi_\text{old}(a_t|s_t)}[\pi_\theta(a_t|s_t)(\mathbb{I}_{a_t=a''_t}-\pi_\theta(a''_t|s_t))(\mathbb{I}_{a_t=a'_t}-\pi_\theta(a'_t|s_t))-\pi_\theta(a_t|s_t)\pi_\theta(a'_t|s_t)(\mathbb{I}_{a'_t=a''_t}-\pi_\theta(a''_t|s_t))].
    \end{aligned}
\end{equation}
According to Equation \eqref{eq:ppo-second-derivative}, the Hessian matrix $\boldsymbol{H}_\text{PPO}$ of $\mathcal{L}_\text{PPO}$ with respect to logits at time step $t$ is given by:
\begin{equation}
    \boldsymbol{H}_\text{PPO}=\frac{A(s_t,a_t)}{\pi_\text{old}(a_t|s_t)}\left[-\boldsymbol{e}^{(k)}\boldsymbol{e}^{(k)\top}+\boldsymbol{\pi}_\theta\boldsymbol{e}^{(k)\top}+\boldsymbol{e}^{(k)}\boldsymbol{\pi}_\theta^\top-2\boldsymbol{\pi}_\theta\boldsymbol{\pi}_\theta^\top+\text{diag}(\boldsymbol{\pi}_\theta)\right],
\end{equation}
where $k$ is the index of action $a_t$ in the vocabulary, $\boldsymbol{e}^{(k)}$ represents the standard $|\mathcal{V}|$-dimension basis vector with a 1 at position $k$, $\boldsymbol{\pi}_\theta\in\mathbb{R}^{|\mathcal{V}|}$ denotes the action probability distribution over the vocabulary $\mathcal{V}$ at time step $t$, and $\text{diag}(\boldsymbol{\pi}_\theta)$ is a diagonal matrix with the elements of $\boldsymbol{\pi}_\theta$ along its main diagonal. For  any vector $\boldsymbol{v}\in\mathbb{R}^{|\mathcal{V}|}$, consider the quadratic form:
\begin{equation}
    \boldsymbol{v}^\top\boldsymbol{H}_\text{PPO}\boldsymbol{v}=\frac{A(s_t,a_t)}{\pi_\text{old}(a_t|s_t)}\left[\underbrace{\sum^{|\mathcal{V}|}_i\pi_{\theta,i}v^2_i-\left(\sum^{|\mathcal{V}|}_i\pi_{\theta,i}v_i\right)^2}_{\mathbb{D}(\boldsymbol{v})}-\underbrace{\left(v^2_k-2v_k\sum^{|\mathcal{V}|}_i\pi_{\theta,i}v_i+\left(\sum^{|\mathcal{V}|}_i\pi_{\theta,i}v_i\right)^2\right)}_{(v_k-\mathbb{E}(\boldsymbol{v}))^2}\right],
\end{equation}
where $\mathbb{E}(\boldsymbol{v})=\sum^{|\mathcal{V}|}_i\pi_{\theta,i}v_i$, and $\mathbb{D}(\boldsymbol{v})=\sum^{|\mathcal{V}|}_i\pi_{\theta,i}v^2_i-(\sum^{|\mathcal{V}|}_i\pi_{\theta,i}v_i)^2$. According to Equation \eqref{eq:sft-psd}, we have $\mathbb{D}(\boldsymbol{v})\geq 0$. Now, consider the case when $A(s_t,a_t)>0$:
\begin{equation}
    \begin{cases}
        \boldsymbol{v}^\top\boldsymbol{H}_\text{PPO}\boldsymbol{v}<0, & \text{ if } v_k>\mathbb{E}(\boldsymbol{v})+\sqrt{\mathbb{D}(\boldsymbol{v})} \text{ or } v_k<\mathbb{E}(\boldsymbol{v}) - \sqrt{\mathbb{D}(\boldsymbol{v})} \\
        \boldsymbol{v}^\top\boldsymbol{H}_\text{PPO}\boldsymbol{v}\geq0, & \text{ otherwise. }
    \end{cases}
\end{equation}
A symmetric result holds for the case where $A(s_t,a_t)<0$. This implies Hessian matrix $\boldsymbol{H}_\text{PPO}$ is not PSD, which indicates that the PPO loss $\mathcal{L}_\text{PPO}$ is not convex with respect to logits.

\subsection{Logits Convexity of LCO Losses}
\label{sec:proof-of-lco-mse-convexity}

\paragraph{Logits Convexity of LCO-MSE Loss} At time step $t$, the Hessian matrix $\boldsymbol{H}_\text{LCO-MSE}$ of $\mathcal{L}_\text{LCO-MSE}$ (defined in Equation \eqref{eq:lco-mse-objective}) with respect to the logits $\boldsymbol{z}_\theta$ is given by:
\begin{equation}
    \boldsymbol{H}_\text{LCO-MSE}=\frac{2}{|\mathcal{V}|}\boldsymbol{I},
\end{equation}
where $\boldsymbol{I}$ is the identity matrix. Since the eigenvalues $\frac{2}{|\mathcal{V}|}>0$, it follows that the Hessian is positive definite. Consequently, $\mathcal{L}_\text{LCO-MSE}$ is not only convex but also strongly convex in the logit space, with a modulus of strong convexity equal to $2/|\mathcal{V}|$.

% For any vector $\boldsymbol{v}\in\mathbb{R}^n$, we consider the quadratic form:
% \begin{equation}
%     \boldsymbol{v}^\top\boldsymbol{H}_\text{LCO-MSE}\boldsymbol{v}=\frac{2}{|\mathcal{V}|}\sum^n_iv^2_i\geq0.
% \end{equation}
% This result confirms that $\boldsymbol{H}_\text{LCO-MSE}$ is PSD, implying that $\mathcal{L}_\text{LCO-MSE}$ is convex with respect to the logits. Furthermore, since $\boldsymbol{v}^\top \boldsymbol{H}_\text{LCO-MSE} \boldsymbol{v} > 0$ for any non-zero vector $\boldsymbol{v}$, the Hessian is strictly positive definite. Consequently, $\mathcal{L}_\text{LCO-MSE}$ is not only convex but also strongly convex in the logit space, with a constant modulus of strong convexity equal to $2/|\mathcal{V}|$.

% \subsection{Logits Convexity of LCO-LCH Loss}
% \label{sec:proof-of-lco-lch-convexity}

\paragraph{Logits Convexity of LCO-LCH Loss} At time step $t$, the partial derivative of $\mathcal{L}_\text{LCO-LCH}$ (defined in Equation \eqref{eq:lco-log-cosh-objective}) with respect to a logit $z_\theta(s_t,a'_t)$ is given by:
\begin{equation}
    \frac{\partial\mathcal{L}_\text{LCO-LCH}}{\partial z_\theta(s_t,a'_t)}=\frac{1}{|\mathcal{V}|}\tanh (z_\theta(s_t,a'_t)-z^*(s_t,a'_t)).
\end{equation}
The corresponding second-order derivative is non-zero only for the diagonal elements ($a'_t = a''_t$):
\begin{equation}
    \frac{\partial^2\mathcal{L}_\text{LCO-LCH}}{\partial z_\theta(s_t,a'_t)\partial z_\theta(s_t,a''_t)}=\frac{1}{|\mathcal{V}|}\text{sech}^2(z_\theta(s_t,a'_t)-z^*(s_t,a'_t)).
\end{equation}
Thus, the Hessian $\boldsymbol{H}_\text{LCO-LCH}$ is given by the diagonal matrix:
\begin{equation}
    \boldsymbol{H}_\text{LCO-LCH}=\text{diag}\left(\frac{1}{|\mathcal{V}|}\text{sech}^2(\boldsymbol{z}_\theta-\boldsymbol{z}^*)\right),
\end{equation}
where $\boldsymbol{z}_\theta\in\mathbb{R}^{|\mathcal{V}|}$ denotes the logits vector, and $\boldsymbol{z}^*\in\mathbb{R}^{|\mathcal{V}|}$ denotes the optimal logits vector.
Since the eigenvalues $\frac{1}{|\mathcal{V}|}\text{sech}^2(x)\in(0,\frac{1}{|\mathcal{V}|}]$ for all $x\in\mathbb{R}$, it follows that the Hessian is positive definite. Furthermore, within any bounded region where $\|\boldsymbol{z}_\theta-\boldsymbol{z}^*\|_{\infty}\leq R$, the eigenvalues of the Hessian are lower-bounded by $\frac{1}{|\mathcal{V}|}\text{sech}^2(R)>0$. Consequently, $\mathcal{L}_\text{LCO-LCH}$ is not only convex but also locally strongly convex at the logits level.

% For any nonzero vector $\boldsymbol{v}\in\mathbb{R}^n$, consider the quadratic form:
% \begin{equation} 
%     \boldsymbol{v}^\top \boldsymbol{H}_\text{LCO-LCH} \boldsymbol{v} = \frac{1}{|\mathcal{V}|} \sum^n_i v^2_i \text{sech}^2(z_{\theta,i} - z^*_i) > 0. 
% \end{equation}
% Consequently, $\boldsymbol{H}_\text{LCO-LCH}$ is PSD, ensuring the logits convexity of the $\mathcal{L}_\text{LCO-LCH}$ objective. Notably, since $\boldsymbol{v}^\top \boldsymbol{H}_\text{LCO-LCH} \boldsymbol{v}>0$ for any non-zero vector $\boldsymbol{v}$, the Hessian is strictly positive definite for finite logits. Then, $\mathcal{L}_\text{LCO-LCH}$ is not only convex but also strongly convex in the logit space.

% \subsection{Logits Convexity of LCO-KLD Loss}
% \label{sec:proof-of-lco-kld-convexity}

\paragraph{Logits Convexity of LCO-KLD Loss} At time step $t$, the partial derivative of $\mathcal{L}_\text{LCO-KLD}$ (defined in Equation \eqref{eq:lco-kld-objective}) with respect to a logit $z_\theta(s_t,a'_t)$ is given by:
\begin{equation}
    \frac{\partial \mathcal{L}_\text{LCO-KLD}}{\partial z_\theta(s_t,a'_t)}=\pi_\theta(a'_t|s_t)-\pi^*(a'_t|s_t).
\end{equation}
And the second derivative is as follow:
\begin{equation}
    \label{eq:lco-second-derivative}
    \frac{\partial^2 \mathcal{L}_\text{LCO-KLD}}{\partial z_\theta(s_t,a'_t)\partial z_\theta(s_t,a''_t)}=\pi_\theta(a'_t|s_t)(\mathbb{I}_{a'_t=a''_t}-\pi_\theta(a''_t|s_t)).
\end{equation}
According to Equation \eqref{eq:lco-second-derivative}, the Hessian matrix $\boldsymbol{H}_\text{LCO-KLD}$ of $\mathcal{L}_\text{LCO-KLD}$ with respect to logits at time step $t$ is given by:
\begin{equation}
    \boldsymbol{H}_\text{LCO-KLD}=\text{diag}(\boldsymbol{\pi}_\theta) - \boldsymbol{\pi}_\theta\boldsymbol{\pi}^\top_\theta,
\end{equation}
From the conclusion in Equation \eqref{eq:sft-psd}, $\boldsymbol{H}_\text{LCO-KLD}$ is PSD, which in turn implies $\mathcal{L}_\text{LCO-KLD}$ is logits convex at each time step.

\section{Upper Bound of Gradient Norm for LCO Objectives}
\label{sec:proof-of-grdient-upper-bound}

\subsection{Upper Bound of LCO-MSE}

Given the gradient norm expression for $\mathcal{L}_\text{LCO-MSE}$ at time step $t$:
\begin{equation}
    \|\nabla_\theta\mathcal{L}_\text{LCO-MSE}\|=\frac{2}{|\mathcal{V}|}\|\boldsymbol{J}^\top(\boldsymbol{z}_\theta-\boldsymbol{z}^*)\|,
\end{equation}
where $\boldsymbol{J}$ is the Jacobian matrix $\nabla_\theta\boldsymbol{z}_\theta^\top$. Under Assumption \ref{ass:second-order-assumption}, $\boldsymbol{J}$ is approximately constant in the local neighborhood. Let $\sigma_\text{max}$ be the maximum singular value of $\boldsymbol{J}^\top$, $\Delta z(s_t,a_t)=z_\theta(s_t,a_t)-z^*(s_t,a_t)$, using property of matrix induced norm:
\begin{equation}
    \|\boldsymbol{J}^\top(\boldsymbol{z}_\theta-\boldsymbol{z}^*)\|\leq\sigma_\text{max}\sqrt{\sum_{a_t\in\mathcal{V}}\Delta z(s_t,a_t)^2}=\sigma_\text{max}\sqrt{|\mathcal{V}|\mathcal{L}_\text{LCO-MSE}}.
\end{equation}
Therefore, the upper bound of the gradient norm is:
\begin{equation}
     \|\nabla_\theta\mathcal{L}_\text{LCO-MSE}\|\leq\frac{2}{|\mathcal{V}|}\sigma_\text{max}\sqrt{|\mathcal{V}|\mathcal{L}_\text{LCO-MSE}}.
\end{equation}
As the loss $\mathcal{L}_\text{LCO-MSE}$ decreases, the upper bound of the gradient norm decreases monotonously.

\subsection{Upper Bound of LCO-LCH}

Given the gradient norm expression for $\mathcal{L}_\text{LCO-LCH}$, at time step $t$:
\begin{equation}
    \|\nabla_\theta\mathcal{L}_\text{LCO-LCH}\|=\frac{1}{|\mathcal{V}|}\left\|\boldsymbol{J}^\top\tanh(\boldsymbol{z}_\theta-\boldsymbol{z}^*)\right\|,
\end{equation}
Using the property of the matrix induced norm:
\begin{equation}
    \left\|\boldsymbol{J}^\top\tanh(\boldsymbol{z}_\theta-\boldsymbol{z}^*)\right\|\leq\sigma_\text{max}\sqrt{\sum_{a_t\in\mathcal{V}}\tanh^2\Delta z(s_t,a_t)}.
\end{equation}
To derive the upper bound, we need to express $\tanh^2$ in terms of $\log\cosh$. We use the following hyperbolic identity:
\begin{equation}
    \sum_{a_t\in\mathcal{V}}\tanh^2\Delta z(s_t,a_t)=\sum_{a_t\in\mathcal{V}}\left[1-\exp{(-2\log\cosh\Delta z(s_t,a_t))}\right].
\end{equation}
By use the Jensen's Inequality for the concave function $1-e^{-2x}$:
\begin{equation}
    \frac{1}{|\mathcal{V}|}\sum_{a_t\in\mathcal{V}}\left[1-\exp{(-2\log\cosh\Delta z(s_t,a_t))}\right]\leq 1-\exp\left(-\frac{2}{|\mathcal{V}|}\sum_{a_t\in\mathcal{V}}\log\cosh\Delta z(s_t,a_t)\right).
\end{equation}
Therefore, the upper bound of the gradient norm is:
\begin{equation}
    \|\nabla_\theta\mathcal{L}_\text{LCO-LCH}\|\leq\frac{1}{|\mathcal{V}|}\sigma_\text{max}\sqrt{|\mathcal{V}|(1-\exp(-2\mathcal{L}_\text{LCO-LCH}))}.
\end{equation}
As the loss $\mathcal{L}_\text{LCO-LCH}$ decreases, the upper bound of the gradient norm decreases monotonously.

\subsection{Upper Bound of LCO-KLD}
Given the gradient norm expression for $\mathcal{L}_\text{LCO-KLD}$ at time step $t$:
\begin{equation}
    \|\nabla_\theta\mathcal{L}_\text{LCO-KLD}\|=\|\boldsymbol{J}^\top(\boldsymbol{\pi}_\theta-\boldsymbol{\pi}^*)\|.
\end{equation}
Let $\Delta\pi(s_t,a_t)=\pi_\theta(s_t,a_t)-\pi^*(s_t,a_t)$, using the property of the matrix induced norm:
\begin{equation}
    \|\boldsymbol{J}^\top(\boldsymbol{\pi}_\theta-\boldsymbol{\pi}^*)\|\leq\sigma_\text{max}\sqrt{\sum_{a_t\in\mathcal{V}}\Delta\pi(s_t,a_t)^2}.
\end{equation}
Since $L_2$ norm is upper-bounded by the $L_1$ norm, according to the Pinsker's Inequality:
\begin{equation}
    \sqrt{\sum_{a_t\in\mathcal{V}}\Delta\pi(s_t,a_t)^2}\leq\sum_{a_t\in\mathcal{V}}|\Delta\pi(s_t,a_t)|\leq\sqrt{2\sum_{a_t\in\mathcal{V}}\pi^*(a_t|s_t)\log\frac{\pi^*(a_t|s_t)}{\pi_\theta(a_t|s_t)}}.
\end{equation}
Therefore, the upper bound of the gradient norm is:
\begin{equation}
    \|\nabla_\theta\mathcal{L}_\text{LCO-KLD}\|\leq\sigma_\text{max}\sqrt{2\mathcal{L}_\text{LCO-KLD}}.
\end{equation}
As the loss $\mathcal{L}_\text{LCO-KLD}$ decreases, the upper bound of the gradient norm decreases monotonously.

\section{Convergence Guarantees of LCO}
\label{sec:lco-convergence-guarantees}
To offer a theoretical guarantee for LCO, we present the following. Let $\mathcal{L}$ be the LCO loss function. Given a step size $\eta$, the parameter update rule is: $\theta_{k+1}=\theta_k-\eta J^\top\nabla_z\mathcal{L}$, where $J=\nabla_\theta z_{\theta}^\top$ denotes the Jacobian matrix. Let $z_k$ be the logits at $k$-th step, under Assumption \ref{ass:second-order-assumption}, we have:
\begin{equation}
    \label{eq:appendix-update}
    z_{k+1}-z_k=J(\theta_{k+1}-\theta_k).
\end{equation}
From the derivation in Appendix \ref{sec:proof-of-lco-mse-convexity}, since $\|\boldsymbol{H}_{LCO-MSE}\|_2\le\frac{2}{|V|}$, $\|\boldsymbol{H}_{LCO-LCH}\|_2\le\frac{1}{|V|}$, and $\|\boldsymbol{H}_{LCO-KLD}\|_2\le1$, using the property of $\beta$-smoothness we have:
\begin{equation}
    \label{eq:appendix-beta-smooth}
    \mathcal{L}(z_{k+1})\le \mathcal{L}(z_k)+\nabla_z \mathcal{L}^\top (z_{k+1}-z_k)+\frac{\beta}{2}\|z_{k+1}-z_k\|^2.
\end{equation}
Combining Equations \eqref{eq:appendix-update} and \eqref{eq:appendix-beta-smooth}, we have:
\begin{equation}
    \label{eq:appendix-beta-smooth-1}
    \mathcal{L}(z_{k+1})\le \mathcal{L}(z_k)-\eta\|J^\top\nabla_z \mathcal{L}\|^2+\frac{\beta\eta^2}{2}\|JJ^\top\nabla_z\mathcal{L}\|^2.
\end{equation}
Since $JJ^\top$ is symmetric and PSD, let $\lambda$ denote its largest eigenvalue. Using spectral decomposition:
\begin{equation}
    \label{eq:appendix-decomposition}
    \|JJ^\top\nabla_z \mathcal{L}\|^2\le\lambda\|J^\top\nabla_z \mathcal{L}\|^2.
\end{equation}
Combining Equations \eqref{eq:appendix-beta-smooth-1} and \eqref{eq:appendix-decomposition}, and setting $\eta=\frac{1}{\beta\lambda}$, we have:
% \begin{equation}
%     \frac{1}{2\beta\lambda}\|J^\top\nabla_z \mathcal{L}\|^2\le \mathcal{L}(z_k)-\mathcal{L}(z_{k+1}).
% \end{equation}
\begin{equation}
    \label{eq:appendix-eq6}
    \frac{1}{\beta\lambda}\|J^\top\nabla_z \mathcal{L}\|^2\le 2(\mathcal{L}(z_k)-\mathcal{L}(z_{k+1})).
\end{equation}
Now, we consider the distance between the current parameters $\theta$ and the optimal parameters $\theta^*$. By expanding the squared norm of the parameter update rule, we have:
\begin{equation}
    \begin{aligned}
        \|\theta_{k+1}-\theta^*\|^2 &= \|\theta_k - \eta J^\top \nabla_z \mathcal{L} - \theta^*\|^2 \\
        &= \|\theta_k-\theta^*\|^2 - 2\eta(\theta_k-\theta^*)^\top J^\top \nabla_z \mathcal{L} + \eta^2 \|J^\top \nabla_z \mathcal{L}\|^2.
    \end{aligned}
\end{equation}
Using the first-order property of convex functions with respect to logits, we have $\mathcal{L}(z^*) \ge \mathcal{L}(z_k) + \nabla_z \mathcal{L}^\top (z^* - z_k)$. By Assumption \ref{ass:second-order-assumption}, we substitute $z^* - z_k = J(\theta^* - \theta_k)$, which yields:
\begin{equation}
    \label{eq:appendix-convexity-property}
    (\theta_k - \theta^*)^\top J^\top \nabla_z \mathcal{L} \ge \mathcal{L}(z_k) - \mathcal{L}(z^*).
\end{equation}
By setting $\eta = \frac{1}{\beta\lambda}$ and substituting Equation \eqref{eq:appendix-convexity-property} back into the distance expansion, we obtain:
\begin{equation}
    \|\theta_{k+1}-\theta^*\|^2 \le \|\theta_k-\theta^*\|^2 - \frac{2}{\beta\lambda}(\mathcal{L}(z_k) - \mathcal{L}^*) + \frac{1}{\beta^2\lambda^2}\|J^\top \nabla_z \mathcal{L}\|^2.
\end{equation}
Incorporating the result from Equation \eqref{eq:appendix-eq6}, the term $\frac{1}{\beta^2\lambda^2}\|J^\top \nabla_z \mathcal{L}\|^2$ is bounded by $\frac{2}{\beta\lambda}(\mathcal{L}(z_k) - \mathcal{L}(z_{k+1}))$. Substituting this bound gives:
\begin{equation}
    \begin{aligned}
        \|\theta_{k+1}-\theta^*\|^2 &\le \|\theta_k-\theta^*\|^2 - \frac{2}{\beta\lambda}(\mathcal{L}(z_k) - \mathcal{L}^*) + \frac{2}{\beta\lambda}(\mathcal{L}(z_k) - \mathcal{L}(z_{k+1})) \\
        &= \|\theta_k-\theta^*\|^2 - \frac{2}{\beta\lambda}(\mathcal{L}(z_{k+1}) - \mathcal{L}^*).
    \end{aligned}
\end{equation}
Summing this inequality from step $k=0$ to $K-1$, we have:
\begin{equation}
    \label{eq:appendix-summation}
    \frac{2}{\beta\lambda} \sum_{k=0}^{K-1} (\mathcal{L}(z_{k+1}) - \mathcal{L}^*) \le \|\theta_0 - \theta^*\|^2 - \|\theta_K - \theta^*\|^2 \le \|\theta_0 - \theta^*\|^2.
\end{equation}
Because Equation \eqref{eq:appendix-eq6} implies that the loss is non-increasing ($\mathcal{L}(z_{k+1}) \le \mathcal{L}(z_k)$), we have:
\begin{equation}
    K(\mathcal{L}(z_K) - \mathcal{L}^*) \le \sum_{k=0}^{K-1} (\mathcal{L}(z_{k+1}) - \mathcal{L}^*).
\end{equation}
Therefore, combining the above results, we obtain the final convergence rate:
\begin{equation}
    \label{eq:appendix-convergence-rate}
    \mathcal{L}(z_K) - \mathcal{L}^* \le \frac{\beta\lambda \|\theta_0 - \theta^*\|^2}{2K}.
\end{equation}
This provides a general convergence guarantee for LCO, demonstrating that it converges at a sublinear rate of $\mathcal{O}(1/K)$ under the given assumptions.

\section{Optimal Logits Analysis}

While the optimal policy is invariant to a constant shift in logits due to the translation invariance of the softmax function, the choice of the specific target $\boldsymbol{z}^*$ in Equation \eqref{eq:optimal-logits} is not arbitrary. We demonstrate that this specific target minimizes the initial logit-space discrepancy, thereby facilitating a more efficient optimization trajectory.

To justify this choice, consider a general candidate target $\boldsymbol{z}'$ defined as $\boldsymbol{z}' = \boldsymbol{z}_{\text{old}} + \frac{\boldsymbol{A}}{\beta} + C$, where $C$ is an arbitrary constant and is added to each element of advantage vector $\frac{\boldsymbol{A}}{\beta}$. This shift $C$ modifies the initial error vector. For $\mathcal{L}_{\text{LCO-MSE}}$ objective, the initial loss at the start of an update is proportional to:
\begin{equation}
    \label{eq:minimization-problem}
    \min_{C} \left\| \boldsymbol{z}_{\text{old}} - \left( \boldsymbol{z}_{\text{old}} + \frac{1}{\beta}\boldsymbol{A} + C \right) \right\|^2 = \min_{C} \left\| \frac{1}{\beta}\boldsymbol{A} + C \right\|^2.
\end{equation}
By setting the derivative with respect to $C$ to zero, we find the optimal shift $C^*$:
\begin{equation}
    \frac{\partial}{\partial C} \sum_{a \in \mathcal{V}} \left( \frac{A(s,a)}{\beta} + C \right)^2 = 0 \implies C^* = -\frac{1}{\beta |\mathcal{V}|} \sum_{a \in \mathcal{V}} A(s,a).
\end{equation}
In modern reinforcement learning, it is standard practice to employ advantage normalization within each time step, ensuring that the mean advantage $\frac{1}{|\mathcal{V}|} \sum_{a_t \in \mathcal{V}} A(s_t,a_t) = 0$. Under this condition, the optimal shift $C^*$ becomes exactly zero. 

This provides a rigorous justification for the definition of $\boldsymbol{z}^*$ in Equation \eqref{eq:optimal-logits}: among all possible logit targets that induce the optimal policy, our choice is the unique target that minimizes the $L_2$ distance from the current behavioral logits. By minimizing this initial discrepancy, we ensure that the parameter updates $\Delta \theta$ are as small as possible to achieve the desired policy shift, which is consistent with the principle of trust region updates. While derived via $\mathcal{L}_{\text{LCO-MSE}}$, this logic extends to $\mathcal{L}_{\text{LCO-LCH}}$ because the Log-Cosh loss is also a symmetric, monotonic function of the residual, thus sharing the same minimum at $C=0$ under normalized advantages.

\end{document}